\documentclass[lettersize,journal]{IEEEtran}

\IEEEoverridecommandlockouts
\usepackage{cite,enumitem}
\usepackage{amsmath,amssymb,amsfonts}
\usepackage{amsthm,comment}
\usepackage{graphicx}
\usepackage{textcomp}
\usepackage{xcolor}
\usepackage[]{footmisc}
\usepackage{bbm}

\newcommand{\nm}[1]{\textcolor{magenta}{\bf [NM: #1]}}
\newcommand{\addb}[1]{\textcolor{black}{#1}}
\newcommand{\addbb}[1]{\textcolor{black}{#1}}

\newcommand{\secref}[1]{Sec.~\ref{#1}}

\usepackage{cuted}
\usepackage{stfloats}
\usepackage{mathtools,amssymb,lipsum, nccmath}
\newtheorem{theorem}{Theorem}
\newtheorem{lemma}{Lemma}

\usepackage{algorithm}
\usepackage{algpseudocode}
\usepackage{color,soul}
\usepackage{graphicx}
\usepackage{mathrsfs}
\usepackage{aligned-overset}
\usepackage{hyperref}
\usepackage{subcaption}
\newtheorem{remark}{Remark}
\theoremstyle{definition}

\newtheorem{Ass}{Assumption}
\theoremstyle{plain}

\usepackage[nodisplayskipstretch]{setspace}
\usepackage{subcaption} 
\newcommand{\kappasc}{\kappa_{\mathrm{sc}}}
\newcommand{\kappanc}{\kappa_{\mathrm{nc}}}

\usepackage{xr}

\def\BibTeX{{\rm B\kern-.05em{\sc i\kern-.025em b}\kern-.08em
    T\kern-.1667em\lower.7ex\hbox{E}\kern-.125emX}}

\begin{document}
\bstctlcite{BSTcontrol}


\title{Biased Federated Learning\\ under Wireless Heterogeneity}
\author{Muhammad Faraz Ul Abrar, \IEEEmembership{Graduate Student Member, IEEE} and Nicol\`{o} Michelusi,~\IEEEmembership{Senior Member, IEEE}
\thanks{M. Faraz Ul Abrar and N. Michelusi are with the School of Electrical, Computer and Energy
Engineering, Arizona State University. email: \{mulabrar,
nicolo.michelusi\}@asu.edu. This research has been funded in part by NSF under grant CNS-$2129015$.}
\thanks{Preliminary versions of this work appeared in 
\cite{Biased_OTA_FL_ICC} and 
\cite{icc_ncvx_biased_ota_fl}.}
\vspace{-5mm}
}

\IEEEaftertitletext{\vspace{-0.25\baselineskip}}

\maketitle
\setulcolor{red}
\setul{red}{2pt}
\setstcolor{red}

\addtolength{\textheight}{0.04in}

\begin{abstract} 
Federated learning (FL) has emerged as a promising framework for distributed learning, enabling collaborative model training without sharing private data. Existing wireless FL works primarily adopt two communication strategies: (1) over-the-air (OTA) computation, which exploits wireless signal superposition for simultaneous gradient aggregation, and (2) digital communication, which allocates orthogonal resources for gradient uploads. Prior work on OTA and digital FL either enforces zero bias (explicitly or via assumed \emph{homogeneous} path loss) or permits uncontrolled bias, yielding high-variance updates under \emph{heterogeneous} channels and creating a performance bottleneck due to devices with poor channel conditions. \addb{We propose wireless FL updates that admit a structured, time-invariant model bias to achieve low-variance gradient aggregation, and analyze their convergence in a unified framework,
 in both strongly convex and non-convex settings.
 }The resulting bounds reveal a bias-variance trade-off governed by the design parameters. To optimize this trade-off, we pose a non-convex joint design problem and develop a successive convex approximation framework to tune the parameters. \addb{Extensive experiments across heterogeneous wireless settings, covering both strongly convex and non-convex image classification tasks, compare the proposed OTA and digital designs against state-of-the-art baselines.} The results demonstrate that optimizing the bias–variance trade-off through a structured bias yields faster FL convergence and improved generalization over existing schemes.
\end{abstract}



\vspace{-3mm}
\section{Introduction}
The surge of massive data generated by Internet-of-Things (IoT) devices, with significant advancements in their computational capabilities, has shifted the focus from classical machine learning (ML) to distributed learning paradigms. Among the distributed learning frameworks, federated learning (FL) has attracted increasing popularity in both academia and industry due to its robust privacy guarantees and reduced communication overhead \cite{FL_survey,DistL_intro}. 
With FL,
$N$ devices (e.g., smartphones and IoT sensors) with private data collaborate with a central parameter server (PS) (e.g., a cloud or edge server)
by exchanging only local model or gradient information,
with the goal to train an ML model 
solving
\begin{align}
\mathbf{w}^* = \arg\min_{\mathbf{w} \in \mathbb{R}^d} F(\mathbf{w}) \triangleq \frac{1}{N} \sum_{m \in [N]} f_m(\mathbf{w}), \tag{P}
\label{FL_prob}
\end{align}
where $f_m(\mathbf{w})$ represents the local objective of device $m$ (e.g., cross-entropy loss) and $F(\mathbf{w})$ is the global objective (loss) function. To solve \eqref{FL_prob}, gradient-based first-order iterative optimization methods such as distributed stochastic gradient descent (SGD) have been widely utilized \cite{Fedavg}. In each FL round, the PS broadcasts the latest FL model to all devices, which then compute their local gradients, 
and send them back to the PS for aggregation. This training process is repeated over several rounds until convergence.

While FL obviates raw data transmission, communication efficiency remains a critical bottleneck in wireless systems due to high-dimensional gradient exchanges over bandwidth-constrained noisy channels \cite{DistL_intro,FL_survey,FL_wireless_impact}. Two primary approaches have emerged to address this: digital FL \cite{Upd_aware_sched,Sched_latency_FL,Sched_policies,Quant_FL_Outage_Constraint,Wireless_Quant_FL_Joint,FL_unreliable_resource_const}, which uses orthogonal resource block (RB) allocation for gradient uploads; 
and over-the-air FL (OTA-FL) \cite{OTA_FL,FL_fading,One_bit_FL,BB_FL,NCOTA,OTA_FL_H_data,Biased_OTA_FL_ICC}, which exploits the natural superposition property of wireless multiple access channels (MAC) and allows simultaneous transmission to realize “single-shot” gradient aggregation. The literature on digital FL focuses primarily on designing communication-efficient device scheduling and RB allocation strategies to accelerate convergence (e.g., \cite{Upd_aware_sched,Sched_policies}). In contrast, OTA-FL works aim to design power control (pre-scaling and post-scaling) strategies to mitigate the noise in the updates \cite{OTA_FL,FL_fading, OTA_FL_H_data}. However, these approaches largely assume wireless homogeneity, where all devices experience the same average path loss, to ensure unbiased FL updates. In practical heterogeneous networks, weaker channel devices act as stragglers, and enforcing a zero-bias design introduces high variance in the FL updates. \addb{While prior works (e.g., \cite{BB_FL,Opt_Power_Control_OTA_Comp,Opt_power_control_OTA_FL}) permit non-zero bias, they do not provide mechanisms to \emph{control} the induced bias or quantify its impact on convergence under wireless heterogeneity. We address this gap by designing OTA and digital FL updates with a \emph{structured, time-invariant} bias governed by explicit design parameters, yielding a tunable bias–variance trade-off. We develop a unified convergence analysis that provides optimality-gap bounds for strongly convex objectives and finite-time stationarity guarantees for smooth non-convex objectives, making the roles of bias and variance explicit. Building on these insights, we pose the joint design as a non-convex problem and devise a successive convex approximation (SCA) framework that operates with statistical CSI at the server and admits efficient implementation. Extensive comparisons with state-of-the-art (SOTA) OTA and digital FL baselines demonstrate consistent gains 
in heterogeneous settings.}

\vspace{-3mm}
\subsection{Related Works and Motivation}
A central challenge in deploying FL over practical wireless networks is maintaining reliable device-PS communication over noisy channels. 
Recent studies have examined how wireless constraints shape FL performance. For example, \cite{Joint_L_Comm} analyzes joint resource-block (RB) allocation and device selection under packet errors, while \cite{FL_over_wireless_optimized} jointly optimizes communication efficiency and bandwidth allocation to accelerate convergence in wireless FL. To mitigate communication overhead, gradient quantization \cite{Prob_Quant1,Prob_Quant2} and sparsification \cite{Sparse_GD,Convergence_Sparse_Methods} have been widely explored. Alternatively, selecting a subset of devices per round \cite{optimal_client_selection_FL,Biased_client_selection} or performing multiple local SGD steps to reduce PS–device communication frequency 
\cite{local_SGD}
has been proposed. However, these studies do not jointly account for communication-efficient FL techniques and wireless impairments such as channel fading and noise.

To address this gap, several works have examined the performance of these schemes over practical wireless networks. Within digital FL, device selection and RB allocation have been addressed using heuristic schemes based on channel state information (CSI) and norm-based local gradient significance \cite{Upd_aware_sched}, while optimization-based device scheduling has been employed to achieve faster convergence in \cite{FL_unreliable_resource_const,ADFL}. Gradient upload costs are further reduced through probabilistic dithered quantization, targeting overall convergence time \cite{Quant_FL_Outage_Constraint} or quantization-induced variance \cite{Wireless_Quant_FL_Joint}. \addb{In the OTA-FL setting, \cite{BB_FL} proposes low-complexity scheduling to balance exploited data against aggregation noise but offers no convergence guarantees and lacks principled bias control. Channel-inversion power control \cite{OTA_FL,CHARLES} enforces unbiased aggregation, yet is limited by the weakest device, amplifying update variance under heterogeneity. To relax this, \cite{Opt_Power_Control_OTA_Comp} minimizes MSE for generic OTA function computation, insightful for aggregation but detached from learning-centric objectives and reliant on global instantaneous CSI at the base station. Closer to FL, \cite{Opt_power_control_OTA_FL} optimizes OTA-FL convergence, but presumes genie-aided global CSI (across all future rounds) and adopts a simplified aggregation rule without PS post-scaling, which can introduce uncontrolled bias in the absence of an explicit zero-bias constraint. A recent comparison \cite{Dig_vs_Analog} contrasts OTA and digital FL under wireless impairments via optimized device sampling while enforcing zero-bias updates, but does not jointly optimize the sampling and communication design parameters.} \addb{In contrast, we develop offline SCA-based designs that minimize the convergence bounds derived herein, requiring only the devices’ large-scale channel conditions at the server.}

Despite these efforts, prior studies either (i) assume wireless homogeneity, so that all devices experience the same average path loss (yielding zero-bias updates), (ii) enforce zero-bias updates under heterogeneity, or (iii) allow biased updates whose bias is \emph{uncontrolled}. In particular, \cite{Upd_aware_sched, FL_fading, OTA_FL, One_bit_FL, OTA_FL_H_data,OTA_FL_Optimization} focus on homogeneous deployments, yielding uniform average participation and thus no model bias, an assumption that is often unrealistic in practice. In contrast, \cite{Dig_vs_Analog, FL_unreliable_resource_const, Sched_policies, Wireless_Quant_FL_Joint, Quant_FL_Outage_Constraint, ADFL, Joint_L_Comm}  allow heterogeneity but mandate unbiased updates to guarantee convergence, which can inflate update variance by accommodating weak devices (the straggler effect). While this phenomenon is well noted for OTA-FL, an analogous bottleneck arises in digital FL, where weak devices demand disproportionate RB allocations under fixed budgets. Finally, while \cite{BB_FL,Opt_Power_Control_OTA_Comp,Opt_power_control_OTA_FL} relax the zero-bias constraint, the induced bias is neither explicitly controlled nor aligned with learning dynamics, and its effect under heterogeneity remains unquantified. \addb{By enabling a fixed, well-structured bias, our framework optimizes the bias-variance trade-off for both OTA and digital FL. Moreover, much of the convergence theory for wireless FL focuses on (strongly) convex objectives (e.g., \cite{OTA_FL_H_data,Opt_power_control_OTA_FL,Biased_OTA_FL_ICC, Dig_vs_Analog,Wireless_Quant_FL_Joint}, which is inconsistent with modern non-convex ML models used in practice. To bridge this gap, we provide a unified theory: optimality-gap bounds for strongly convex objectives and finite-time stationarity guarantees for smooth non-convex objectives. Extensive experiments demonstrate consistent gains over SOTA baselines in heterogeneous wireless deployments.}

\vspace{-3mm}
\subsection{Contributions and Organization}
Extending our prior works \cite{Biased_OTA_FL_ICC,icc_ncvx_biased_ota_fl}, we study OTA and digital FL over heterogeneous wireless networks. Our main contributions are:
\begin{enumerate}[leftmargin=*, labelsep=0.5em]
\item We introduce OTA and digital FL updates that admit a \emph{fixed, time-invariant} model bias parameterized by explicit design variables. Unlike prior zero-bias or uncontrolled-bias approaches, the proposed approach enables principled variance reduction via a tunable bias–variance trade-off.

\item \addb{We develop a unified convergence analysis that (i) applies to both OTA and digital schemes with the same bias–variance structure (differing only in the variance term), and (ii) covers both objective classes, optimality-gap bounds for strongly convex objectives, and finite-time stationarity bounds for non-convex objectives. In all cases, the bounds explicitly decompose bias and variance as functions of the design parameters.}

\item 
To optimize the bias-variance trade-off for faster convergence, we pose non-convex joint parameter designs \addb{using only large-scale channel conditions at the PS} and solve them via SCA. 

\item We benchmark the proposed designs against several SOTA OTA/digital baselines and show that our optimized designs consistently converge faster and generalize better on both strongly convex and non-convex image-classification tasks.
\end{enumerate}

The rest of this paper is organized as follows:  \secref{Sec:System_model} presents the system model and the biased OTA and digital FL schemes. \secref{Sec:Convergence_Analaysis} presents the theoretical convergence analysis (proofs in the Appendix). 
\secref{Sec:Optimal_Design} discusses comprehensive optimization-based frameworks for biased OTA-FL and digital-FL parameter design. Numerical results are detailed in \secref{Sec:Numerical_Results}, followed by concluding remarks in  \secref{Sec:Conclusion}.

\vspace{-3mm}
\subsection{Notation} 
A boldface lowercase letter represents a vector, e.g., $\mathbf v$. $\mathcal{CN}(\mathbf m,\Sigma)$ is a  circularly symmetric complex Gaussian random vector with
mean $\mathbf m$ and covariance $\Sigma$. The operators $\left\Vert \mathbf v \right\Vert$, $\left\Vert \mathbf v \right\Vert_\infty$, and $\mathbf v^\top$ denote the $\ell_2$-norm, $\ell_\infty$-norm, and transpose of $\mathbf v$, respectively. $[N]$ denotes the discrete set $\{1,2,\dots,N\}$. The expectation of a random variable over its associated probability distribution is denoted by $\mathbb{E}[\cdot]$. For a random vector $\mathbf v$, we denote its variance as $\mathrm{var}(\mathbf v){=}\mathbb E[\Vert\mathbf v{-}\mathbb E[\mathbf v]\Vert^2]$ ($\mathrm{var}(\mathbf v|\mathcal F)$ when conditioned on $\mathcal F$). For a collection of variables $\mathcal X$, $\mathcal X \ge 0$ denotes elementwise nonnegativity.

\section{System Model and Wireless FL}
\label{Sec:System_model}

We consider a wireless network with $N$ distributed devices and a single base station, acting as a parameter server (PS), that collaboratively learn a model via FL, as shown in Fig.~\ref{fig:FL_system_model}.
Each device $m \in [N]$ owns a private local dataset $\mathcal{D}_m$ 
and a private local objective $ f_m(\mathbf{w}) \triangleq \frac{1}{|\mathcal{D}_m|}\sum_{\boldsymbol{\xi}\in\mathcal{D}_m}\phi(\mathbf{w},\boldsymbol{\xi}),$
where $\phi(\mathbf{w},\boldsymbol{\xi})$ denotes the sample-wise loss evaluated at data sample $\boldsymbol{\xi}$ and $\mathbf{w}\in\mathbb{R}^d$ is the learning parameter. We assume balanced local dataset sizes, i.e., $|\mathcal{D}_m|=D$ for all $m$, so that the global objective $F(\mathbf w)$ in \eqref{FL_prob} is equivalent to uniform weighting over all samples across the network. \footnote{Our framework readily extends to imbalanced datasets by using a weighted global objective $F(\mathbf w)=\sum_{m\in[N]} \frac{|\mathcal D_m|}{\sum_{n \in [N]}|\mathcal D_n|}\,f_m(\mathbf w)$. We focus on the balanced case to isolate the impact of wireless heterogeneity.}
In this work, we employ the distributed SGD method over multiple FL rounds to solve \eqref{FL_prob}. At the start of round $t$, the PS broadcasts the latest model parameter $\mathbf{w}_t$ to each device in the network. Next, device $m$ uses a randomly drawn mini-batch  $\mathcal{B}_{m,t} \subseteq \mathcal{D}_m$ to estimate its local gradient $\boldsymbol{g}_{m,t} = \frac{1}{\vert \mathcal{B}_{m,t} \vert} \sum_{\boldsymbol{\xi} \in \mathcal{B}_{m,t}} \nabla \phi(\mathbf{w}_t,\boldsymbol{\xi})$, with $\mathbb{E} [\boldsymbol{g}_{m,t}] = \nabla f_m(\mathbf{w}_t)$, where $\nabla f_m(\mathbf{w}_t)$ is the full-batch local gradient. Next, each device uploads its estimated local gradient to the PS. Ideally, the PS aims to compute the true global stochastic gradient
 \begin{align}
\overline{\boldsymbol{g}}_t = \frac{1}{N}\sum_{m \in [N]} \boldsymbol{g}_{m,t},
\label{Gradagg}
\end{align}
obtained by aggregating the local gradients from each device without any errors. With it, the PS updates the FL model as
\begin{align}
	    \mathbf{w}_{t+1} = \mathcal{P}_{\mathcal{W}} \left(\mathbf{w}_{t} - {\eta} \overline{\boldsymbol{g}}_t\right),  \quad t \geq0\,,\label{GD}
\end{align}
where $\eta$ represents the learning step size, and $\mathcal{P}_{\mathcal{W}}(\cdot)$ denotes projection onto a closed and convex set $\mathcal{W}$ such that $\mathbf{w}^* \in \mathcal{W}$. This projection ensures compliance with practical constraints, such as privacy and energy-limited transmissions \cite{Clip_DPFL}.
The FL updates in \eqref{GD} are repeated until a desired metric, such as accuracy, is achieved or a fixed number of learning rounds $T$ is completed. Nevertheless, \eqref{Gradagg} requires noiseless aggregation of all the local gradients, each contributing a fraction $1/N$ of the total. In practice, however, the
gradient aggregation is affected by errors due to imperfect communications induced by noisy wireless fading channels.
We model the wireless channel between each device $m \in [N]$ and the PS as a Rayleigh flat block fading channel, i.e., $h_{m,t} \sim \mathcal{CN}(0, \Lambda_m)$. The channel coefficients are independent and identically distributed (i.i.d.) over FL rounds and remain constant within a single round. The parameter $\Lambda_m$ represents the average 
channel gain, dependent on large-scale propagation conditions, and is assumed constant throughout FL runtime and known to the PS.
\footnote{The PS can obtain $\Lambda_m$ at the start of the learning procedure. 
}

We emphasize that, unlike existing works 
\cite{Upd_aware_sched,FL_fading,OTA_FL,One_bit_FL,OTA_FL_H_data} assuming identical average path loss across devices ($\Lambda_m{=}\Lambda_n , \forall m, n {\in} [N]$), we consider a heterogeneous wireless environment where devices experience varying path losses. We consider two widely studied communication schemes: 1) OTA computation
\cite{OTA_FL,FL_fading,BB_FL,One_bit_FL,NCOTA,OTA_FL_H_data,Biased_OTA_FL_ICC}
and 2) digital transmission-based FL \cite{Upd_aware_sched,Sched_latency_FL,Sched_policies,Quant_FL_Outage_Constraint,Wireless_Quant_FL_Joint,FL_unreliable_resource_const}. 
 Similar to \cite{OTA_FL,OTA_FL_H_data,BB_FL,Sched_policies,Upd_aware_sched} and other related works, the downlink broadcast transmission of the FL model is assumed to be noiseless. Hence, we focus solely on the uplink communication model. 

\addb{Before presenting these  schemes in 
\secref{subsecOTA} and \secref{subsecDIGI}, respectively, we first present two operating assumptions next.}

\begin{figure}
\includegraphics[scale=0.25]{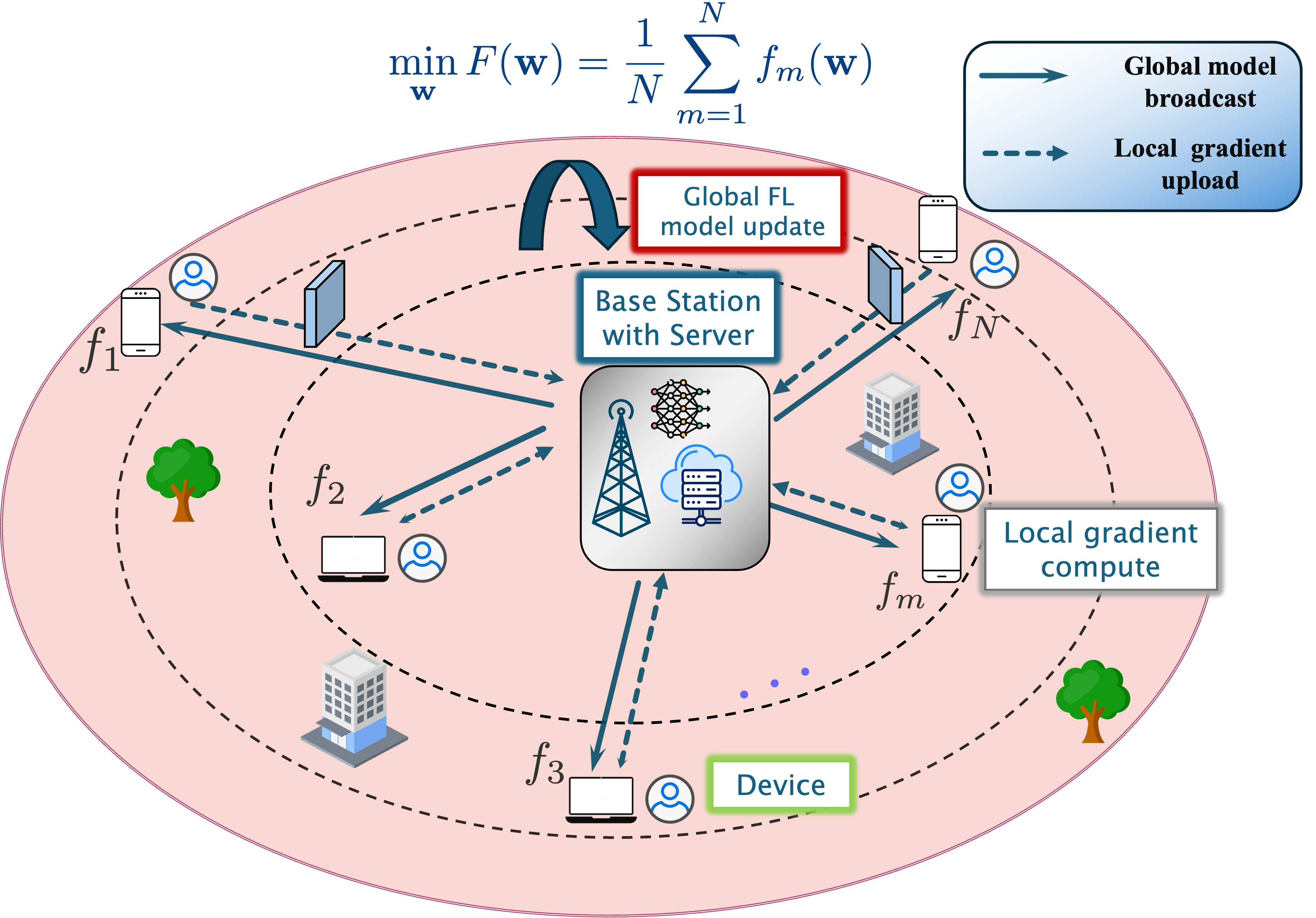}
\centering
\caption{A wireless FL setup with one parameter server collaborating with $N$ devices with heterogeneous wireless conditions.}
\vspace{-7mm}
\label{fig:FL_system_model}
\end{figure}

\begin{Ass}  
\label{ass:bounded_loss_grad}
The sample-wise loss gradient for any given individual data sample $\xi$ is bounded, i.e., $\Vert \nabla \phi(\mathbf{w}, \xi) \Vert \leq G_\text{max},\forall \mathbf{w}\in\mathcal W$. It then follows from the triangular inequality that $\Vert\mathbf g_{m,t}\Vert\leq G_\text{max},\forall m,t$.
\end{Ass}  
\begin{Ass} 
\label{ass:bounded_stochastic_grad}
The mini-batch stochastic local gradient $\boldsymbol{g}_{m,t}$ is an unbiased estimate of the full-batch local gradient with bounded variance, i.e., $\mathbb{E}[\boldsymbol{g}_{m,t}|\mathbf w_t] = \nabla f_m(\mathbf{w}_t)$ and $
\mathrm{var}(\boldsymbol{g}_{m,t}|\mathbf w_t) \leq \sigma_m^2, \, \forall m \in [N], \, \mathbf{w}_t \in \mathcal{W}, t \geq 0$. \end{Ass}


\noindent\addb{Assumption \ref{ass:bounded_loss_grad} is widely adopted by many wireless FL works, see e.g., \cite{Transmission_Power_Control_OTA_FA,OTA_FL_H_data,OTA_FL_Optimization,Wireless_Quant_FL_Joint}, while Assumption \ref{ass:bounded_stochastic_grad} is standard in SGD-based FL analyses (e.g., \cite{niid_fedavg,OTA_FL_H_data,Opt_power_control_OTA_FL,OTA_FL_Optimization,Quant_FL_Outage_Constraint}).}

\subsection{Over-the-air transmission}
\label{subsecOTA}
OTA-FL exploits the superposition property of the wireless MAC to perform joint computation and communication, enabling “one-shot’’ aggregation of local gradients at the PS \cite{Analog_WSN}. \addb{
Following standard practice in the OTA-FL literature, we assume perfect time (and carrier) synchronization across devices during uplink aggregation \cite{OTA_FL,FL_fading,BB_FL,One_bit_FL,NCOTA,OTA_FL_H_data,Biased_OTA_FL_ICC}. This assumption can be relaxed via recently developed calibration techniques for OTA aggregation; see, e.g., \cite{OTA_realignment,ngo2025distributedmimoovertheairphase}.}
To perform OTA-FL, each device $m$ 
maps its local gradient to a transmit signal $\mathbf{x}_{m,t}$, 
and devices transmit simultaneously, hence the
 PS receives
\begin{align}
\mathbf{y}_t \;=\; \sum_{m\in[N]} h_{m,t}\,\mathbf{x}_{m,t} \;+\; \mathbf{z}_t,
\label{Signal_model}
\end{align}
where $\mathbf{z}_t \sim \mathcal{CN}(\mathbf 0, N_0 \mathbf I)$ is the additive white Gaussian noise at the PS, i.i.d. over $t$. To approximate the ideal aggregation in \eqref{Gradagg} under an average per-sample energy constraint $E_s$, devices employ truncated channel inversion with a device-specific pre-scaling factor \(\gamma_m\):
\begin{align}
\mathbf{x}_{m,t} \;=\; \frac{1}{h_{m,t}}\,\chi^A_{m,t}\,\gamma_m\,\boldsymbol{g}_{m,t},
\label{OTA_signal_transmission}
\end{align}
where $\mathbf{\chi}^A_{m,t}$ is the OTA transmission indicator, defined as 
\begin{align}
\mathbf{\chi}^A_{m,t} =  
\begin{cases}
1, \,\quad \text{ if } \vert h_{m,t}\vert \geq \frac{G_\text{max} \gamma_m}{\sqrt{d E_s}},\\
0,\quad \,\,\,\text{otherwise}. \label{OTA_indicator}
\end{cases}
\end{align}
Here $G_\text{max}$ is an upper bound on $\Vert\boldsymbol{g}_{m,t}\Vert$ (see Assumption~\ref{ass:bounded_loss_grad}), ensuring \(\|\mathbf{x}_{m,t}\|^2/d \le E_s\). Note that a device does not participate in round $t$ if $\vert h_{m,t}\vert < \frac{G_\text{max} \gamma_m}{\sqrt{d E_s}}$. Notably, the participation rule is decentralized, requiring only local instantaneous CSI \(h_{m,t}\), which can be obtained via a downlink pilot under channel reciprocity \cite{FL_d2d}. \addb{Unlike homogeneous OTA-FL designs that use a common pre-scaler or a common threshold, our scheme allows device-specific pre-scalers $\{\gamma_m\}$ and hence device-specific thresholds $\{\frac{G_\text{max} \gamma_m}{\sqrt{d E_s}}\}$, tailored to their specific average path loss conditions $\Lambda_m$. Importantly, these are not a fixed heuristic: pre-scalers (and thus the thresholds) are optimized \emph{once offline} (Sec.\ref{OTA_FL_optimization_design}) using only statistical CSI, then held time-invariant during training while \(\chi^A_{m,t}\) adapts online to the instantaneous channel.} 
With this design, the PS estimates the stochastic global gradient \eqref{Gradagg} as
\begin{align}
\hat{\boldsymbol{g}}_t = \frac{\mathbf{y}_t} {\alpha}
=\frac{1}{\alpha} \sum_{m \in [N]}\chi^A_{m,t}\gamma_{m}\boldsymbol{g}_{m,t}
\;+    
\;\frac{\mathbf{z}_t}{\alpha},
\label{OTA_Signal_model2}
\end{align}
where $\alpha$ is a post-scaler. To provide intuition for our choice of the global gradient in \eqref{OTA_Signal_model2}, observe that,
taking expectation over fading and noise,
 $\mathbb{E} [\mathbf{y}_t|\{\boldsymbol{g}_{m,t}\}_m] = \sum_{m \in [N]} \alpha_{m} \boldsymbol{g}_{m,t}$, where $\alpha_{m} =  {\gamma_{m} \exp\{ \frac{ - \gamma_{m}^2 G_\text{max}^2}{d \Lambda_m  E_s}}\}$. By setting the post-scaler as $\alpha = \sum_{m \in [N]} \alpha_{m}$, the estimated global gradient $\hat{\boldsymbol{g}}_t$ satisfies a desirable property: the expected estimate $\tilde{\boldsymbol{g}}_t \triangleq \mathbb{E}[\hat{\boldsymbol{g}}_t|\{\boldsymbol{g}_{m,t}\}_m]$ is a convex combination of the stochastic local gradients, i.e.,
\begin{align}
\tilde{\boldsymbol{g}}_t = \sum_{m \in [N]} p_m
 \boldsymbol{g}_{m,t},
 \label{OTA_exp_global_grad} 
\end{align}
where $p_m \triangleq \alpha_m/\alpha$ represents the average \emph{participation level} of device $m$ induced by the OTA-FL design and satisfies $0 \leq p_m \leq 1$ and $\sum_{m \in [N]} p_{m}{=}1$. 
Thus, $\hat{\boldsymbol{g}}_t$ is an unbiased estimator of $\tilde{\boldsymbol{g}}_t$, but a \emph{biased} estimator of the target global gradient $\overline{\boldsymbol{g}}_t$ in \eqref{Gradagg}, since $\tilde{\boldsymbol{g}}_t$ allows non-uniform participation ($p_m$ vs. $1/N$).
The implications of such non-uniform participations are discussed in \secref{subsec:Biased_FL}.
By further taking the expectation with respect to the random mini-batch data selection, we obtain
\begin{align}
\mathbb E[\tilde{\boldsymbol{g}}_t|\mathbf w_t] = \sum_{m \in [N]} p_m \nabla f_m(\mathbf w_t). \label{Biased_global_grad}
\end{align}
We next characterize the variance of the OTA-FL global gradient estimation error with respect to its expected value \eqref{Biased_global_grad}; the proof is provided in Appendix B.
\begin{lemma}
\label{OTA_variance_lemma}
Under Assumptions \ref{ass:bounded_loss_grad} and \ref{ass:bounded_stochastic_grad}, the OTA-FL global gradient estimation variance satisfies
$\mathrm{var}(\hat{\boldsymbol{g}}_t|\mathbf w_t)\leq\zeta^A$, with
\begin{align*}
&\zeta^A{\triangleq}
\underbrace{\sum_{m \in [N]} p_{m}^2 G^2_\text{max} \left(\frac{\gamma_{m}}{\alpha_{m}}{-}1 \right)}_{\text{transmission variance}}
+\!\!\!\!\!\!\!\!\!\!\underbrace{\sum_{m \in [N]} p_{m}^2\sigma_m^2}_{\text{mini-batch gradient variance}}
\!\! +\!\!\!\!\!\underbrace{\frac{d N_0}{\alpha^2}}_{\text{noise variance}}\!\!\!\!.
\end{align*}
\end{lemma}
We note three terms in the above Lemma:  (1) transmission variance, arising from intermittent local transmissions induced by the threshold-based strategy in \eqref{OTA_indicator}, where devices with poor channel conditions may not transmit; (2) mini-batch gradient variance,
due to random mini-batch selections; and (3) noise variance, stemming from the additive noise at the PS.

Due to concurrent uplink transmissions by the devices, the overall gradient upload time in each OTA-FL round is $\tau = \frac{d}{B}$, independent of the number of devices, where $B$ denotes the communication bandwidth.
\subsection{Digital Transmission}
\label{subsecDIGI}
We now describe the digital uplink model for solving \eqref{FL_prob}, following \cite{Upd_aware_sched,Sched_latency_FL,Sched_policies,Quant_FL_Outage_Constraint,Wireless_Quant_FL_Joint,FL_unreliable_resource_const}. We employ time-division multiple access (TDMA): in each FL round, every participating device is assigned an orthogonal time slot for gradient upload. To reduce communication overhead, local gradients are quantized prior to transmission. Specifically, device \(m\in[N]\) normalizes its gradient as \(\boldsymbol{g}_{m,t}/\|\boldsymbol{g}_{m,t}\|_\infty\), then quantizes each normalized entry using \(r_m\) bits via the dithered stochastic uniform quantizer \cite{Prob_Quant1,Prob_Quant2}, with \(r_m\) fixed throughout training. The device transmits its quantized normalized gradient together with the gradient norm, for a total payload of \(L_m = 64 + d\,r_m\) bits.

To mitigate uplink delay under deep channel fading, we adopt a thresholded participation rule:
\begin{align}
\chi^D_{m,t} \;=\;
\begin{cases}
1, & \text{if } |h_{m,t}| \ge \rho_m,\\
0, & \text{otherwise},
\end{cases}
\label{Digital_indicator}
\end{align}
where \(\rho_m\) is a device-specific threshold. Devices decide participation locally using instantaneous CSI \(h_{m,t}\) (e.g., estimated from a downlink pilot).
\addb{We assume device \(m\) throughout transmits at a fixed data rate $R_m 
\;=\; B\log_2\!\Big(1 + \tfrac{E_s\,\rho_m^2}{N_0}\Big)$
 where \(B\) is the allocated bandwidth,
  \(E_s\) is the average transmit energy per symbol and \(N_0\) is the noise power spectral density. With the rule \(|h_{m,t}|\ge\rho_m\) in \eqref{Digital_indicator}, the instantaneous signal-to-noise ratio (SNR) exceeds $\frac{E_s\,\rho_m^2}{N_0}$, ensuring an outage-free link at data rate $R_m$. When the channel cannot support \(R_m\), the devices simply refrain from participation in a given round.}

Using these orthogonal local gradient transmissions by the devices, the PS estimates the stochastic global gradient as
\begin{align}
\hat{\boldsymbol{g}}_t \;=\; \sum_{m \in [N]} \frac{\chi^D_{m,t}\,\boldsymbol{g}_{m,t}^q}{\nu_m},
\label{Digital_Signal_model2}
\end{align}
where \(\boldsymbol{g}_{m,t}^q\) denotes the reconstruction (at the PS) of device \(m\)’s dithered-quantized gradient, \(\nu_m\) is a device-specific post-scaler at the PS.\footnote{Unlike OTA-FL, orthogonal transmissions in digital FL allow device-specific post-scalers, providing an additional degree of freedom.}
Leveraging the unbiasedness of dithered stochastic uniform quantization \cite{Prob_Quant1,Prob_Quant2,Quant_FL_Outage_Constraint},
we obtain
\begin{align}
\tilde{\boldsymbol{g}}_t \triangleq \mathbb{E}[\hat{\boldsymbol{g}}_t|\{\boldsymbol{g}_{m,t}\}_m]
 =\sum_{m \in [N]} \frac{\beta_m}{\nu_m}\,\boldsymbol{g}_{m,t}
{\triangleq}\!\!\! \sum_{m \in [N]} p_m\,\boldsymbol{g}_{m,t},
\label{Digital_exp_global_grad}
\end{align}
where \(\beta_m \triangleq \mathbb{E}[\chi^D_{m,t}] = \exp\{-\rho_m^2/\Lambda_m\}\) and $p_m \triangleq \beta_m/\nu_m \,,\forall m \in [N]$. As in OTA-FL, we interpret \(p_m\) as the average (digital FL) device participation level and enforce \(0\le p_m\le 1\) and \(\sum_{m} p_m=1\) as a design constraint so that \(\tilde{\boldsymbol{g}}_t\) satisfies \eqref{Biased_global_grad}. The variance of the global-gradient estimation error is characterized next; the proof is provided in Appendix~B.
\begin{lemma}
\label{Dig_variance_lemma}
Under Assumptions \ref{ass:bounded_loss_grad} and \ref{ass:bounded_stochastic_grad}, the digital-FL global gradient estimation variance satisfies 
$\mathrm{var}(\hat{\boldsymbol{g}}_t|\mathbf w_t)\leq\zeta^D$, with
\begin{align*}
\nonumber
&\zeta^D\triangleq
\underbrace{\sum_{m \in [N]} p_{m}^2 G^2_\text{max} \left(\frac{1}{\beta_{m}} - 1 \right)}_{\text{transmission variance}}
\nonumber\\&
+ \!\!\!\underbrace{\sum_{m \in [N]} p_{m}^2 \sigma_m^2}_{\text{mini-batch gradient variance}}
\!\!\!+ \underbrace{\sum_{m \in [N]} p_m^2 G_\text{max}^2 {} \frac{d }{\beta_m(2^{r_m} - 1)^2}}_{\text{quantization noise variance}}.
\end{align*}
\end{lemma}
Similarly to OTA-FL, the variance bound $\zeta^D$ is decomposed into three components: (1) transmission variance, due to intermittent transmissions following the threshold-based approach in \eqref{Digital_indicator}; (2) mini-batch gradient variance, due to mini-batch sampling; and (3) quantization noise variance, due to quantization of mini-batch local gradients.

The uplink latency in round $t$ for device $m$ is
$\tau_{t,m}=\chi_{m,t}^D \frac{L_m}{B R_m}$ ($=0$ if it does not participate with \(\chi^D_{m,t}=0\)). Using \(\mathbb{E}[\chi^D_{m,t}]=\beta_m\), the expected per-round latency is:  
\begin{align}  
\mathbb{E}\Big[\sum_{m \in [N]} \tau_{t,m}\Big] 
= \sum_{m \in [N]} \frac{\beta_m L_m}{B R_m}. \label{Expected_Dig_lat}  
\end{align}  

\subsection{Biased FL}
\label{subsec:Biased_FL}
With the gradient estimates in \eqref{OTA_Signal_model2} and \eqref{Digital_Signal_model2}, the PS updates the model as
\begin{align}
\mathbf{w}_{t+1} \;=\; \mathcal{P}_{\mathcal{W}}\!\left(\mathbf{w}_t - \eta\,\hat{\boldsymbol{g}}_t\right),  \quad t \geq0\,,
\label{Biased_GD}
\end{align}
where $\mathcal{P}_{\mathcal{W}}(\cdot)$ denotes projection onto the feasible set \(\mathcal{W}\).
In both OTA and digital FL schemes, \(\hat{\boldsymbol{g}}_t\) is an unbiased estimator of
\(\tilde{\boldsymbol{g}}_t \triangleq \sum_{m\in[N]} p_m\,\boldsymbol{g}_{m,t}\), so \eqref{Biased_GD} is a noisy-SGD step with \(\tilde{\boldsymbol{g}}_t\) replacing the ideal average gradient \(\overline{\boldsymbol{g}}_t\) in \eqref{Gradagg}. Consequently, the updates minimize, on average, a \emph{re-weighted} objective
\begin{align}
\tilde{F}(\mathbf{w}) \;\triangleq\; \sum_{m\in[N]} p_m\, f_m(\mathbf{w}),
\label{Incons_obj}
\end{align}
rather than the global objective \(F(\mathbf{w})\) in \eqref{FL_prob}. In fact, $\mathbb{E}[\tilde{\boldsymbol{g}}_t|\mathbf w_t ]= \nabla \tilde{F}(\mathbf{w}_t)$ (c.f. \eqref{Biased_global_grad}) with expectation taken over the mini-batch data selection. This induces a \emph{model bias} with respect to \(F\), whose precise form depends on the objective class (strongly convex vs. non-convex), as formalized in Sec.~\ref{Sec:Convergence_Analaysis}.

\begin{remark}  
Prior works on OTA/digital FL (e.g., 
\cite{Dig_vs_Analog,FL_unreliable_resource_const,Sched_policies,Wireless_Quant_FL_Joint,Quant_FL_Outage_Constraint,ADFL,Joint_L_Comm,Upd_aware_sched,FL_fading,OTA_FL,One_bit_FL,OTA_FL_H_data}) either assume wireless homogeneity or enforce a zero-bias strategy, ensuring uniform participation $ p_m = \frac{1}{N}$, so that minimizing \eqref{Incons_obj} becomes equivalent to \eqref{FL_prob}. While effective under homogeneous conditions, both schemes suffer from devices with poor channel quality in heterogeneous settings: in OTA-FL, the worst-channel device becomes the bottleneck (as shown in \cite{BB_FL,Biased_OTA_FL_ICC,Opt_power_control_OTA_FL}), while in digital FL, such devices induce a \textit{straggler effect}, dominating latency under constrained communication resources. \end{remark} 
\vspace{-5mm}
\addb{\begin{remark}
The proposed framework incorporates a fixed, time-invariant bias through $\{p_m\}$, controlled by design variables optimized offline
(OTA: device pre-scaling and PS post-scaling; Digital: thresholds, post-scaling, and quantization bits). The convex-combination estimators in \eqref{OTA_exp_global_grad} and \eqref{Digital_exp_global_grad} provide \emph{structured} control over $\{p_m\}$, and hence over the model bias, in contrast to prior \emph{unstructured} biased schemes where the bias implicitly varies with instantaneous CSI and lacks theoretical guarantees. Uniform participation $p_m{=}\frac{1}{N}$ is a special case of our design.
\end{remark}}

These insights motivate biased OTA and digital designs under wireless heterogeneity. As mentioned above, recent studies (e.g., \cite{Opt_power_control_OTA_FL,OTA_FL_Optimization} for OTA-FL)  consider a biased FL design. Yet, they treat a generic, unstructured bias that offers limited control during training. Building on our prior work \cite{Biased_OTA_FL_ICC} and \cite{icc_ncvx_biased_ota_fl}, we instead adopt a \emph{structured, time-invariant} bias that enables tractable analysis. The proposed framework thus allows a non-zero average bias in the FL updates and exposes a bias-variance trade-off to be jointly optimized. \secref{Sec:Optimal_Design} develops the optimization of this trade-off, leveraging convergence bounds derived next for both strongly convex and non-convex settings.
\section{Convergence Analysis}
\label{Sec:Convergence_Analaysis}
In this section, we theoretically study the convergence behavior of the presented FL schemes. Since both schemes follow the update rule in \eqref{Biased_GD}, we adopt a unified convergence framework, where average device participation levels are given by $p_m{=}\frac{\alpha_m}{\alpha}$ or $p_m{=} \frac{\beta_m}{\nu_m}$, and the variance of the global gradient estimation is captured in Lemmas \ref{OTA_variance_lemma} and \ref{Dig_variance_lemma}, for the OTA and digital schemes, respectively. 
\addb{We now state the assumptions used to establish convergence for the two cases of interest, namely, the strongly convex case and the non-convex case.}
\begin{Ass}[Both cases] 
\label{ass:smooth_lb} Each local objective function  $f_m(\cdot)$ 
is $L$-smooth, that is, for all $m \in [N]$, $f_m$ satisfies
$$
\|\nabla f_m(\mathbf{x})-\nabla f_m(\mathbf{y})\|\le L\|\mathbf{x}-\mathbf{y}\|,\quad \forall\,\mathbf{x},\mathbf{y}\in\mathbb{R}^d,
$$
and is lower bounded, that is, there exists $f_m^{\inf}\in\mathbb{R}$ such that $f_m(\mathbf{w})\ge f_m^{\inf}\,,\forall\, \mathbf{w}\in\mathbb{R}^d$. 
Consequently, any convex combination $\sum_{m\in[N]} p_m f_m(\cdot)$ (including $F$ with $p_m=1/N$ and $\tilde F$) is $L$-smooth and lower bounded by $\sum_{m\in[N]} p_m f_m^{\inf}$.
\end{Ass}

\begin{Ass}[Strongly convex case]
\label{ass:sc}
Each local objective function  $f_m(\cdot)$ 
is $\mu$-strongly convex, that is 
$$f_m(\mathbf{y}) \geq f_m(\mathbf{x}) + \nabla f_m(\mathbf{x})^\top(\mathbf{y} - \mathbf{x}) + \frac{\mu}{2}\| \mathbf{y} - \mathbf{x} \|^2,  
$$
for all $\mathbf{x},\mathbf{y} \in \mathbb{R}^d$.  Hence, any convex combination \(\sum_{m \in [N]} p_m f_m(\cdot)\) (including \(F\) and \(\tilde F\)) is \(\mu\)-strongly convex.
\end{Ass}
\vspace{-6mm}
\addb{\begin{Ass}[Non-convex case] 
\label{ass:bounded_data_hetr}
The variance of local gradients with respect to the global gradient is bounded, that is, there exists $\kappanc >0 $ such that $\frac{1}{N}\sum_{m \in [N]}
 \Vert \nabla f_m(\mathbf{w})-\nabla F(\mathbf{w})\Vert^2 \leq \kappanc^2,\forall \,\mathbf{w} \in \mathbb{R}^d$. Under Assumption \ref{ass:bounded_loss_grad}, it further follows that $\kappanc \leq 2 G_\text{max}$. \end{Ass}
Assumptions \ref{ass:smooth_lb} (smoothness) is standard in FL analyses (e.g., \cite{niid_fedavg,OTA_FL_H_data,OTA_FL_Optimization}), with lower-boundedness particularly typical for non-convex objectives (e.g., \cite{One_bit_FL,OTA_FL_Optimization,Quant_FL_Outage_Constraint}). Assumption~\ref{ass:sc} was adopted in \cite{niid_fedavg,OTA_FL_H_data,OTA_FL_Optimization}. Finally, Assumption \ref{ass:bounded_data_hetr} (also called bounded gradient dissimilarity or data divergence) is also widely adopted in non-convex optimization analyses (e.g., \cite{P_FL_bounded_niid,Quant_FL_Outage_Constraint}).}

\begin{remark}  
While prior works \cite{niid_fedavg,OTA_FL_H_data,Upd_aware_sched} assume uniform boundedness of local gradients over $\mathbb R^d$, this assumption contradicts the strong convexity of local objectives, as noted in 
\cite{SGD_Hogwild}. The projection step in our FL updates
resolves this discrepancy by
ensuring $\mathbf{w}_t \in \mathcal{W}$, thereby requiring sample-wise gradient boundedness only over $\mathcal{W}$ (Assumption \ref{ass:bounded_loss_grad}). This condition is easily satisfied in practice, e.g., for smooth losses. \end{remark}


Let \(\mathbf w^\star \in \arg\min_{\mathbf w} F(\mathbf w)\) and \(\tilde{\mathbf w}\in \arg\min_{\mathbf w} \tilde F(\mathbf w)\). In the strongly convex case, these are unique global minimizers; in the non-convex case, they represent stationary solutions. \addb{We measure convergence by the \emph{optimality error} \(\mathbb E[\|\mathbf w_t-\mathbf w^\star\|^2]\) in the strongly convex case, and by the (finite-time) \emph{average stationarity} \(\frac{1}{T}\sum_{t=0}^{T-1}\mathbb E[\|\nabla F(\mathbf w_t)\|^2]\) in the non-convex case.}

We are now ready to present the main convergence results; full proofs are detailed in Appendix A.


\begin{theorem}[Strongly convex case]
\label{thm:main_convergence_sc} 
Under 
Assumptions \ref{ass:bounded_loss_grad}, \ref{ass:bounded_stochastic_grad}, \ref{ass:smooth_lb}, and \ref{ass:sc}, a fixed learning step size $\eta \in \Big(0, \frac{2}{\mu + L}\Big]$, and $\mathbf{w}_0 \in \mathcal{W}$,
and with $\mathcal W\equiv\{\mathbf w\in\mathbb R^d:\Vert\mathbf w\Vert\leq D/2\}$ and $D\triangleq2\max_{m\in[N]}\frac{1}{\mu} \Vert \nabla f_m(\mathbf 0) \Vert$,
the optimality error after $t$ FL rounds satisfies
\begin{equation*}
\begin{split}
&\mathbb{E}[\Vert\mathbf{w}_t - \mathbf{w}^*\Vert^2] \leq \underbrace{2D^2\left(1-\eta \mu\right)^{2t}}_{\text{initialization error}} \\
& + \underbrace{2\frac{N {\kappasc}^2}{\mu^2}\sum_{m \in [N]} \left(\frac{1}{N} - p_{m}\right)^{2}}_{\text{model bias}} \;\;+ \!\!\!\!\underbrace{2\frac{\eta}{\mu} \zeta,}_{\text{gradient estimation variance}}
\end{split}
\end{equation*}
where \(\zeta\) is provided in Lemmas \ref{OTA_variance_lemma} and \ref{Dig_variance_lemma} for OTA and digital schemes, respectively and 
$\kappasc^2\triangleq\frac{1}{N}\sum_{m \in [N]} \left\| \nabla f_m(\mathbf{w}^*)\right\|^2$. 
\end{theorem}
\addb{
\begin{theorem}[Non-convex case]
\label{thm:main_convergence_ncvx}
Under Assumptions \ref{ass:bounded_loss_grad}, \ref{ass:bounded_stochastic_grad}, \ref{ass:smooth_lb}, and \ref{ass:bounded_data_hetr}, and a fixed learning step size $\eta \in \left(0,\frac{1}{L}\right]$ and $\mathbf{w}_0 \in \mathcal{W}$ with $\mathcal{W} \equiv \mathbb{R}^d$, after $T$ FL rounds it holds that
\begin{align*}
    \frac{1}{T}\sum_{t=0}^{T-1}\mathbb{E}\big[\|\nabla F(\mathbf{w}_t)\|^2\big]
    \le & \underbrace{\frac{4\,\max_{m\in[N]} (f_m(\mathbf{w}_0) - f_m^{\inf})}{\eta\,T}}_{\text{initialization error}} \nonumber \\
     &\hspace{-30mm}+ \underbrace{2N\kappanc^2 \sum_{m \in [N]}  \Big(p_m-\frac{1}{N}\Big)^2}_{\text{model bias}}   +  \underbrace{2 \eta L \zeta}_{\text{gradient estimation variance}} ,
\end{align*}
with $\zeta$ given in Lemmas~\ref{OTA_variance_lemma}-\ref{Dig_variance_lemma}.
\end{theorem}
}
The FL convergence bounds in Theorem \ref{thm:main_convergence_sc} and \ref{thm:main_convergence_ncvx} for the strongly convex and non-convex case, respectively, characterize the behavior of the proposed biased OTA-FL and digital FL through three key terms: (1) initialization error, (2) model bias, and (3) global gradient estimation variance. The FL initialization error term is standard and decays geometrically in the strongly convex case and as \(O\!\left(1/(\eta T)\right)\) in the non-convex case. The model-bias term arises due to the (possibly non-uniform) device participation levels $\{p_m\}$. The parameter $\kappasc/\kappanc$ in the bias term quantifies the degree of data heterogeneity across devices. Therefore, the bias vanishes under either uniform participation $(p_m=1/N)$ or when the devices' local objectives are identical $(\kappasc/\kappanc=0)$. The gradient estimation variance captures the noisy estimate of the global gradient via OTA/digital communication, detailed in Lemmas \ref{OTA_variance_lemma} and 
\ref{Dig_variance_lemma}. \addb{Crucially, both strongly convex and non-convex regimes share the same bias-variance structure, unifying the analysis across OTA and digital schemes and across objective classes.} Unlike prior work, our framework makes the impact of a \emph{structured, tunable} bias explicit through the $\{p_m\}$-dependent term, revealing a bias-variance trade-off to accelerate convergence. This trade-off calls for careful optimization of associated design parameters, developed in the next section.

\vspace{-3mm}
\section{Optimal Biased FL Design}
\label{Sec:Optimal_Design}
The bounds in Theorems~\ref{thm:main_convergence_sc} and \ref{thm:main_convergence_ncvx} yield important design insights. For OTA-FL, while decreasing $\{\gamma_m\}$ reduces transmission variance and model bias, it leads to noise amplification. Conversely, minimizing noise variance may lead to larger model bias due to non-uniform device participation. Similarly, for digital FL, enforcing uniform participation (zero bias) by designing $\rho_m$ and $\nu_m$ can worsen quantization noise variance and FL round latency, whereas minimizing quantization noise can introduce large bias. These trade-offs motivate a joint optimization of design parameters for improved FL convergence, developed in this section. \addb{Since the strongly convex and non-convex regimes share the same bias-variance structure, we adopt a unified parameter design framework. Notably, our optimization uses only \emph{statistical} (large-scale) CSI and is performed \emph{once offline} before training, avoiding the per-round global CSI acquisition and/or optimization assumed in several existing schemes \cite{Upd_aware_sched,Sched_policies,Wireless_Quant_FL_Joint,BB_FL,OTA_FL,Opt_power_control_OTA_FL}.}

\vspace{-2mm}
\subsection{OTA-FL optimization
}
\label{OTA_FL_optimization_design}
\addb{For the OTA-FL parameter design, 
exploiting the dependence among variables of interest, the minimization of the convergence bound is stated as
}
\begin{subequations}
\begin{align}
&\min_{\{\gamma_m\}, \{p_m\}, \alpha} \quad 
\addb{\omega_{\mathrm{var}}}
\Bigg( \sum_{m \in [N]} p_m^2 G_\text{max}^2 \left( \frac{\gamma_m}{\alpha p_m} - 1 \right)
+ \frac{d N_0}{\alpha^2} \nonumber\\ 
&\quad + 
\sum_{m \in [N]} p_m^2 \sigma_m^2
\Bigg)
+ 
\addb{\omega_{\mathrm{bias}}} \sum_{m \in [N]} \left( \frac{1}{N} - p_m \right)^2, \label{OTA_objective}\\
&\text{s.t.}\,\bullet \gamma_m \, \exp\Big\{-\frac{\gamma_m^2 G_\text{max}^2}{d \Lambda_m E_s}\Big\} = \alpha p_m, \quad \forall m \in [N], \label{OTA_alpha_m}\\
&\quad\;\bullet 0 \le \gamma_m \le \gamma_{m,\text{max}}, \quad \forall m \in [N], \label{OTA_gamma_pos}\\
&\quad\;\bullet 0 \le \alpha \le \min_{m \in [N]} \frac{\alpha_{m, \text{max}}}{p_m}, \label{OTA_alpha_max}\\
&\quad\;\bullet 0 \leq p_m \leq 1, \, \forall m \in [N],\ \sum_
{m'\in[N]} p_{m'} = 1,\label{OTA_p_pos} 
\end{align}
\label{OTA_re_wrriten_P}
\end{subequations}

\noindent corresponding to the convergence bounds from Theorems~\ref{thm:main_convergence_sc}-\ref{thm:main_convergence_ncvx} (strongly convex and non-convex cases, respectively), with the variance term $\zeta$ given in Lemma~\ref{OTA_variance_lemma}, and 
without the initialization error since it does not affect the optimizer. \addb{Here, $\omega_{\mathrm{var}}$ and $\omega_{\mathrm{bias}}$ in \eqref{OTA_objective} weight the variance and bias terms, respectively.}\footnote{\addb{For the strongly convex case, $(\omega_{\mathrm{var}},\,\omega_{\mathrm{bias}})\equiv\big(\eta/\mu,\; N\,\kappasc^{2}/\mu^{2}\big)$, whereas for the non-convex case, $(\omega_{\mathrm{var}},\,\omega_{\mathrm{bias}})\equiv\big(\eta L,\; N\,\kappanc^{2}\big)$.}}
 The optimization problem in \eqref{OTA_re_wrriten_P} jointly optimizes $\{\gamma_m\}$, $\{p_m\}$, and $\alpha$. Specifically, the constraint \eqref{OTA_alpha_m} arises from the definitions of $\alpha_m=\alpha p_m$.
 Next, observe that $\alpha_m$ in \eqref{OTA_alpha_m} is quasi-concave in $\gamma_m$, with its maximum given by $\alpha_{m, \text{max}} = \sqrt{\frac{d \Lambda_m E_s}{2e G_\text{max}^2 }}$, hence, the constraint \eqref{OTA_alpha_max} ensures $\alpha_m \leq \alpha_{m, \text{max}}$ for all $m \in [N]$. 
Since \eqref{OTA_objective} is increasing in $\gamma_m$ and  \eqref{OTA_alpha_m} admits two roots $\gamma_{m,1}\le \gamma_{m,\text{max}}\le \gamma_{m,2}$, where $\gamma_{m,\text{max}} \triangleq \arg \max_{\gamma_m} \alpha_m(\gamma_m) = \sqrt{\frac{d \Lambda_m E_s}{2 G_\text{max}^2}}$, it is without loss of optimality to restrict $\gamma_m\le \gamma_{m,\text{max}}$ as in \eqref{OTA_gamma_pos}. Finally, \eqref{OTA_p_pos} constrains $\{p_m\}$ to the probability simplex, guaranteeing a well-controlled and structured model bias. 

The objective is non-convex, due to the first and third terms in \eqref{OTA_objective} and the nonlinear constraints \eqref{OTA_alpha_m}–\eqref{OTA_alpha_max}.
Furthermore, the differing scales of these terms can make the problem ill-conditioned. First-order methods (e.g., projected gradient descent~\cite{Nesterov_book}) can thus perform suboptimally.
To address this, we adopt a successive convex approximation (SCA) approach \cite{Inner_approx,MM_comms,SCA_book}: we iteratively solve convex surrogates of the original problem by linearizing the non-convex components around the current iterate. This process is guaranteed to converge to a stationary point of the original non-convex problem. 

To this end, we iteratively convexify the problem for $k = 0, 1, \dots, K-1$ by linearizing around the current iterate anchors $\{\overline{\gamma}_m\}$, $\{\overline{p}_m\}$, and $\overline{\alpha}$ at iteration $k$.\footnote{For brevity, we omit the dependence on iteration index $k$.}
To convexify at iteration $k$, first, we reformulate \eqref{OTA_objective} using an epigraph transformation by introducing auxiliary variables $\{z_m\}$ such that $\frac{p_m \gamma_m}{\alpha} \leq z_m$ for all $m \in [N]$, and linearize the concave term $-p_m^2$ around $\overline{p}_m$. Next, we obtain a convex relaxation of the new constraints $\frac{p_m \gamma_m}{\alpha} \leq z_m$ by taking logarithms and linearizing $\ln p_m$ and $\ln\gamma_m$ around $(\overline{p}_m,\overline{\gamma}_m)$, yielding \eqref{OTA_z_epi}. For \eqref{OTA_alpha_m}, we first relax the equality to an inequality, take logarithms, and then linearize around $(\overline{\alpha},\overline{p}_m)$ to obtain a convex constraint \eqref{OTA_alpha_m_conv}. Finally, the bound in \eqref{OTA_alpha_max} is expressed as $\max_{m \in [N]} \frac{p_m}{\alpha_{m, \text{max}}} \leq \frac{1}{\alpha}$, with the right-hand side linearized around $\overline{\alpha}$ to get \eqref{OTA_alpha_max_conv}. 
These steps yield a convex surrogate of the optimization problem in \eqref{OTA_re_wrriten_P} solved at each SCA iteration:
\begin{subequations}
\begin{align}
 &\min_{\{\gamma_m\},  \{p_m\}, \{z_m\}, \alpha }    \addb{\omega_{\mathrm{var}}}\Bigg(\sum_{m \in [N]}  G_\text{max}^2    \,z_m \!+ \! \frac{ d N_0}{\alpha^2} \!+\!\!\!\!\sum_{m \in [N]}p_m^2\sigma_m^2 \nonumber \; \\ -  &\!\!\!\sum_{m \in [N]} \!\!G_\text{max}^2 \overline{p}_m(2 p_m - \overline{p}_m )\Bigg) \!+\! \addb{\omega_{\mathrm{bias}}}\!\!\!\sum_{m \in [N]}\!\!\left(p_m{-}\frac{1}{N} \right)^2\!\!\!,\label{OTA_objective_cvx}\\
&\text{s.t.}\ \forall m \in [N]:\nonumber\\&
\bullet\ 
\ln(\overline{\gamma}_m \overline{p}_m)+ \frac{\gamma_m}{\overline{\gamma}_m}+ \frac{p_m}{\overline{p}_m}-2 \leq \ln z_m + \ln \alpha,
\label{OTA_z_epi}\\
&  \bullet
\ln(\overline{\alpha}\ \overline{p}_m)+ \frac{\alpha}{\overline{\alpha}}+ \frac{p_m}{\overline{p}_m}-2 \leq \ln \gamma_m - \frac{ \gamma_{m}^2 G_\text{max}^2}{d \Lambda_m  E_s} \label{OTA_alpha_m_conv},
\\
&  \bullet 
\frac{p_m}{\alpha_{m, \text{max}}} \leq \frac{2\overline{\alpha}-\alpha}{(\overline{\alpha})^2} \,,
\alpha \geq 0 ,\label{OTA_alpha_max_conv} \\
& \bullet \eqref{OTA_gamma_pos}, \eqref{OTA_p_pos} \nonumber.
\end{align} 
\label{OTA_convexified} 
\end{subequations}
This problem can be efficiently solved using numerical solvers such as CVX \cite{cvx}. The original problem \eqref{OTA_re_wrriten_P} is then tackled by successively solving \eqref{OTA_convexified} with updated linearization anchors: initialize $\{\overline{\gamma}_m^{(0)}\}$, $\{\overline{p}_m^{(0)}\}$, and $\overline{\alpha}^{(0)}$ (e.g., via a low-complexity heuristic; two choices will be listed next), solve \eqref{OTA_convexified}, set the solution as $(\overline{\gamma}_m^{(k+1)},\overline{p}_m^{(k+1)},\overline{\alpha}^{(k+1)})$, and iterate for $K$ rounds.

We highlight that the proposed SCA framework generalizes our earlier heuristics, minimum noise variance, and zero-bias minimum noise variance, proposed in \cite{Biased_OTA_FL_ICC}. For algorithmic details of these designs, we refer the reader to our prior work \cite{Biased_OTA_FL_ICC}.
\label{OTA_FL_optimization_design}
\vspace{-3mm}
\subsection{Digital FL Optimization}
We now optimize the digital-FL parameters by minimizing the convergence bound (c.f. Theorems \ref{thm:main_convergence_sc} and \ref{thm:main_convergence_ncvx}) with the variance term from Lemma~\ref{Dig_variance_lemma}. Let $\mathcal{X}\!\triangleq\!\{\{\rho_m\},\{\beta_m\},\{p_m\},\{\nu_m\},\{r_m\},\{R_m\}\}$ denote the collection of nonnegative design variables. Exploiting the dependencies among variables and dropping the initialization term, we write:
\begin{subequations}
\label{Dig_design}
\begin{align}
\min_{\mathcal{X}\geq 0} &
\quad \addb{\omega_{\mathrm{var}}} \Bigg(
\sum_{m \in [N]} p_{m}^2 G^2_\text{max} \left(\frac{1}{\beta_{m}} - 1 +\frac{d }{\beta_m(2^{r_m} - 1)^2}\right)
\nonumber\\ &\quad
+ \sum_{m \in [N]} p_m^2 \sigma_m^2\Bigg)+ \addb{\omega_{\mathrm{bias}}} \sum_{m \in [N]} \left( \frac{1}{N} - p_m \right)^2,
\label{Dig_design_obj} \\ 
& \text{s.t.}\ \forall m \in [N]:\nonumber\\&
\bullet \sum_{m' \in [N]} \frac{(64 + d r_{m'})}{B R_{m'}} \beta_{m'} \leq T_\text{max}, \label{FL_delay_constraint} \\
& \bullet R_m=\log_2\!\big(1\!+\!\tfrac{E_s \rho_m^2}{N_0}\big),
\beta_m=\exp\{-\tfrac{\rho_m^2}{\Lambda_m}\}, \label{Dig_beta_R_m_beta_m_def}\\
& \bullet p_m = \frac{\beta_m}{\nu_m}\,\,, 0 \leq p_m \leq 1\,,\sum_{m'\in[N]} p_{m'} = 1, \label{Dig_p_m_def_p_sum_1}
\\
& \bullet r_m \in \{1, 2, \dots\}\,, \label{r_int}
\end{align}
\end{subequations}
\addb{where, recall, $\omega_{\mathrm{var}}$ and $\omega_{\mathrm{bias}}$ are the weights of the variance and bias terms, respectively, as in \eqref{OTA_objective}. 
} Here, the latency constraint \eqref{FL_delay_constraint} follows from the expected per-round delay (c.f. \eqref{Expected_Dig_lat}), and \eqref{Dig_beta_R_m_beta_m_def} captures the SNR-rate-threshold coupling. Constraints in  \eqref{Dig_p_m_def_p_sum_1} follow from the definitions of $\beta_m$ and $p_m$ and ensure that $\{p_m\}$ lie on the probability simplex to control model bias. 
Finally, \eqref{r_int} enforces feasibility for choosing the number of bits $r_m$ to quantize local gradients. The resulting optimization problem in \eqref{Dig_design} is mixed-integer and highly non-convex, so we employ an SCA approach. Using \eqref{Dig_beta_R_m_beta_m_def} and \eqref{Dig_p_m_def_p_sum_1} to write $\beta_m=p_m\nu_m$ and $\rho_m=\sqrt{-\Lambda_m\ln(p_m\nu_m)}$, we reduce the search to $\mathcal{X}'=\{\{p_m\},\{\nu_m\},\{r_m\},\{R_m\}\}$. 
We then introduce auxiliary variables $\{z_m\}$ and $\{\varpi_m\}$ with 
$\frac{p_m}{\nu_m}\le z_m$ and $\frac{p_m}{\nu_m(2\cdot 2^{r_m'}-1)^2}\le \varpi_m$, 
and linearize the concave term $-p_m^2$ around $\overline p_m$, yielding the 
convexified objective \eqref{Dig_objective_cvx}. The constraints 
$\frac{p_m}{\nu_m}\le z_m$ and $\frac{p_m}{\nu_m(2\cdot 2^{r_m'}-1)^2}\le \varpi_m$ 
are convexified via a log transform and first-order expansions of $\ln p_m$ around 
$\overline p_m$, giving \eqref{Dig_z_epi} and \eqref{quantization_noise_conv}. We relax \eqref{r_int} by optimizing over continuous $r_m'$ and set $r_m=\lfloor r_m' \rfloor + 1$ post-optimization. For the delay constraint, we introduce $\{t_m\}$ with $\frac{(64+d(r_m'+1))\,\nu_m p_m}{B R_m}\le t_m$ (\eqref{Dig_T_max}) and convexify by taking logs and linearizing $\ln(64+d(r_m'+1))$, $\ln\nu_m$, and $\ln p_m$ around $(\overline r_m',\overline \nu_m,\overline p_m)$ to obtain \eqref{FL_delay_constraint_conv}. The rate constraint in \eqref{Dig_beta_R_m_beta_m_def} is relaxed using ${\rho_m^2=-\Lambda_m\ln(p_m\nu_m)}$ and linearizations of $\ln p_m$ and $\ln\nu_m$, producing \eqref{Dig_R_m_conv}. Finally, since $\beta_m\le 1$ implies $\nu_m\le 1/p_m$, we enforce the linearized bound around $\overline{p}_m$ \eqref{Dig_v_m_conv}. These steps lead to the following convex surrogate in \eqref{Dig_convexified} solved at each SCA iteration:
\begin{subequations}
\begin{align}
&\min_{\mathcal{X}' \ge 0}  \addb{\omega_{\mathrm{var}}}\Big(\sum_{m \in [N]}  
G_\text{max}^2 (z_m + d \varpi_m) + \sum_{m \in [N]} p_m^2 \sigma_m^2
\nonumber\\
&-\!\!\!\sum_{m \in [N]}\!\!\!G_\text{max}^2 \overline{p}_m(2 p_m{-} \overline{p}_m )\Big){+}\addb{\omega_{\mathrm{bias}}}\!\!\sum_{m \in [N]}\!\!\!\Big(p_m{-}\frac{1}{N} \Big)^2\!\!,\label{Dig_objective_cvx}\\
& \text{s.t.}\ \forall m \in [N]:\nonumber\\&
\bullet \ln \overline{p}_m + \frac{p_m - \overline{p}_m}{ \overline{p}_m} \leq \ln {z}_m + \ln{\nu_m}, \label{Dig_z_epi}\\
& \bullet \ln \overline{p}_m{+}\frac{p_m {-}\overline{p}_m}{ \overline{p}_m} \leq \ln {\varpi}_m {+}\ln{\nu_m}{+}2 \ln (2 \cdot 2^{r_m' }{-}1), \label{quantization_noise_conv} \\
& \bullet \ln \overline{\nu}_m + \ln (64 + d + d\overline{r}_m') + \ln{\overline{p}_m} + \frac{\nu_m - \overline{\nu}_m}{ \overline{\nu}_m }  
\nonumber\\&
+\frac{d (r_m' -\overline{r}_m' )}{64 + d + d\overline{r}_m'} + \frac{p_m - \overline{p}_m}{ \overline{p}_m}\ \! \!\leq 
\ln (t_m) +\ln (R_m B), \label{FL_delay_constraint_conv}\\
& \bullet 2^{R_m} \leq 1 - \frac{\Lambda_m E_s }{N_0} \Big(\ln \overline{\nu}_m{+}\frac{\nu_m}{ \overline{\nu}_m}{+}\ln \overline{p}_m{+}\frac{p_m}{\overline{p}_m}{-}2
\Big), \label{Dig_R_m_conv}\\
& \bullet \sum_{m' \in [N]} t_{m'} \leq  T_\text{max}, \label{Dig_T_max}\\
& \bullet \;0 \leq \nu_m \leq  \frac{2\overline{p}_m - p_m} { \overline{p}_m^2},\label{Dig_v_m_conv}\\
&\bullet \;0 \leq p_m \leq 1, \quad \sum_{m'\in[N]} p_{m'} = 1.
\end{align} 
\label{Dig_convexified} 
\end{subequations}
The surrogate problem in \eqref{Dig_convexified} is convex and can be solved efficiently, e.g., via CVX \cite{cvx}, following a similar procedure described for the OTA-FL problem in \eqref{OTA_convexified}.  

\label{Dig_FL_optimization_design}
\vspace{-2mm}
\section{Numerical Results}
\label{Sec:Numerical_Results}
In this section, we perform numerical experimentation to evaluate the performance of our proposed schemes in both strongly convex and non-convex settings. 
We study two image classification problems on the widely used MNIST \cite{MNIST} and \addb{CIFAR-10} \cite{cifar10} datasets with $C$ = 10 classes.
We consider a wireless FL system with \addb{$N\in\{10,50\}$} devices uniformly deployed in a disk of radius 
$\varrho_{\max} = 1750$
m, with the PS at the center.  \addbb{Specifically, device locations are drawn i.i.d.\ uniformly over the disk area. For each device $m$, we sample a polar angle $\theta_m\in[0,2\pi)$ and a radius $s_m\in[0,\varrho_{\max}]$ according to
$\theta_m\sim\mathrm{Unif}[0,2\pi)$ and $s_m=\varrho_{\max}\sqrt{U_m}$ with $U_m\sim\mathrm{Unif}[0,1]$, independently across $m$. Given $s_m$, we compute the large-scale channel gain $\Lambda_m=\mathbb{E}[|h_{m,t}|^2]$ via the log-distance path-loss model $
PL(s_m)=PL_0 + 10\Omega\log_{10}(s_m/s_0)\ \text{(dB)}$,
with path-loss exponent $\Omega=2.2$, and $PL_0=50$ dB at a reference distance $s_0 = 1$ m. We then set
$\Lambda_m = 10^{-PL(s_m)/10}$.
Small-scale channel fading and PS noise are drawn independently across rounds.
}
The communication bandwidth is $B$ = 1 MHz with carrier frequency $f_c = 2.4$ GHz, and the transmission power is set to $P_\text{tx} =$ 0 dBm. The noise power spectral density at the PS is 
$N_0 = -173$ 
dBmW/Hz.  


For FL tasks on both datasets, we use the regularized cross-entropy loss function $\phi$ to define each local objective $f_m$. To emulate a practical FL scenario,
each device holds a local dataset with a limited number of samples (1000 for MNIST, and 100 for CIFAR-10). We consider two challenging non-i.i.d.\ partitions (specified in the corresponding subsections): \emph{single-class per device} and \emph{two-classes per device}.
Because each device sees only a narrow set of classes, these extreme partitions make cross-device collaboration necessary for accurate classification. With limited samples at each device, each device computes the gradient using its full dataset, i.e., $|\mathcal{B}_{m,t}| = |\mathcal{D}_m|, \forall t$, resulting in no mini-batch gradient variance ($\sigma_m^2 = 0$ for all $m \in [N]$) in our simulations. \addbb{Reported curves show mean $\pm$ standard deviation over multiple Monte Carlo trials with independent channel fading and noise realizations, for a fixed device deployment (i.e., fixed $\{\Lambda_m\}$).}
Throughout, step sizes for all schemes are tuned via a small grid search. We organize results by objective class to mirror the theory, with both the strongly convex and non-convex cases under the two wireless FL schemes, as presented next.
\vspace{-4mm}
\subsection{Strongly Convex FL Task}
\label{subsec:strongly_convex_results}
To verify Theorem \ref{thm:main_convergence_sc}, we perform softmax regression (single linear layer) on the MNIST \cite{MNIST} dataset with ten classes. We consider a challenging \emph{single-class per device} data-heterogeneous (non-i.i.d.) setting, where each device's local dataset consists of all the samples of only one class. The FL model parameter is $d {=} 7850$-dimensional with $\mathbf{w}^\top = \begin{bmatrix} {\mathbf{w}^{(0)}}^\top, \cdots, {\mathbf{w}^{(9)}}^\top \end{bmatrix}$. Here, $\mathbf{w}^{(\ell)}$ is the sub-parameter associated with class $\ell$, for $\ell = 0, \cdots, 9$. The per-sample loss is

\begin{align*}
    \phi(\mathbf{w}, (\boldsymbol{x}, \ell)) = \frac{\mu}{2} \Vert \mathbf{w} \Vert^2 - \ln \Big( \frac{\exp{\{ \boldsymbol{x}^\top \mathbf{w}^{(\ell)} \}}}{\sum_{c = 0}^9 \exp{\{ \boldsymbol{x}^\top \mathbf{w}^{(c)} \}}} \Big). \end{align*}
With this, each local objective function $f_m$ is $\mu$-strongly convex and $L = 2 + \mu$-smooth \cite{NCOTA}. Next, we report two sets of results, one per communication modality (OTA and digital).

\subsubsection{Comparison with State-of-the-Art (SOTA) OTA-FL Schemes}
To demonstrate the effectiveness of our analysis, we compare the proposed OTA-FL framework with several SOTA OTA-FL schemes, adapted to our settings to ensure a fair evaluation. For details, we refer to the respective papers.
$\bullet$ \underline{\textit{Optimized Power Control: OTA Computation (OPC OTA-}} \underline{\textit{Comp)}} \cite{Opt_Power_Control_OTA_Comp}, minimizes the MSE distortion for an OTA-based sum computation task by optimizing the pre-scalers $\{\gamma_m\}$ and PS post-scaler $\alpha$.
The power control design for optimal $\{\gamma_m\}$ and $\alpha$ requires \emph{global instantaneous CSI in each FL round}, unlike our proposed scheme requiring only local instantaneous CSI.\\
$\bullet$ \underline{\textit{Low-Complexity Power Control: OTA Computation}},
\underline{\textit{(LCPC OTA-Comp)}} \cite{Opt_Power_Control_OTA_Comp}. It is a low-complexity scheme that follows a truncated channel inversion OTA power control, where all devices use the same tunable pre-scaler.
LCPC OTA-Comp optimizes the MSE, averaged with respect to channel fading, and hence does not require global instantaneous CSI for power control design.
\\
$\bullet$ \underline{\textit{Optimized Power Control: OTA-FL (OPC OTA-FL)}} \cite{Opt_power_control_OTA_FL}. 
It simplifies the OTA-FL design by considering only the device pre-scaler (no PS post-scaler). Assuming CSI knowledge of \emph{all future rounds}, it solves an optimization problem over $\{\gamma_m\}$ to minimize the FL sub-optimality gap over $T$ rounds. For this reason, we label it as \emph{genie-aided}. Notably, OPC OTA-FL does not impose a zero-bias design constraint. \\
$\bullet$ \underline{\textit{Vanilla OTA-FL}} \cite{OTA_FL} is the classical channel inversion-based OTA power control strategy. By assigning the same pre-scaler to each device, Vanilla OTA-FL ensures zero instantaneous bias. However, it requires  \emph{global instantaneous CSI in each FL round at the PS} to design the common pre-scaler.\\
$\bullet$ \underline{\textit{BB-FL Interior}} \cite{BB_FL}, is a low-complexity scheme that
schedules only the devices within a chosen radius $\varrho_\text{in} < \varrho_\text{max}$ to participate in OTA-FL. The participating devices employ truncated channel inversion to upload their local gradients. \\
$\bullet$ \underline{\textit{BB-FL Alternative}} \cite{BB_FL}, is a low-complexity scheme enabling participation of both cell-edge devices with weak average channel gains and cell-interior devices in FL training. It achieves this by randomly alternating between full device participation (scheduling every device) and the BB-FL Interior policy and uses truncated channel inversion power control.

\begin{figure*}[t]
  \centering
  \begin{subfigure}[t]{0.4\textwidth}
    \centering
    \includegraphics[width=\linewidth]{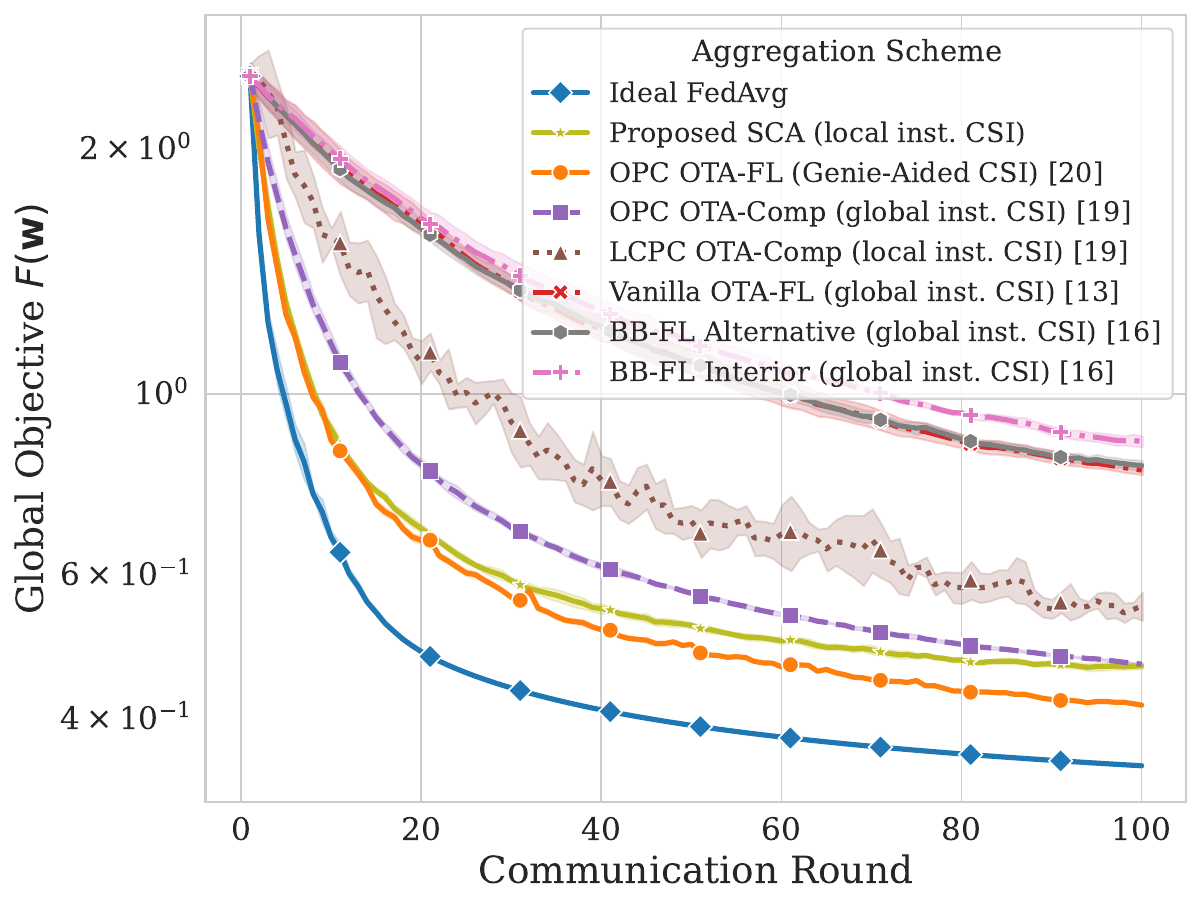}
    \caption{OTA-FL: global objective $F(\mathbf{w})$ vs. FL rounds}
    \label{fig:sc-ota-loss}
  \end{subfigure}
  \quad 
  \begin{subfigure}[t]{0.4\textwidth}
    \centering
    \includegraphics[width=\linewidth]{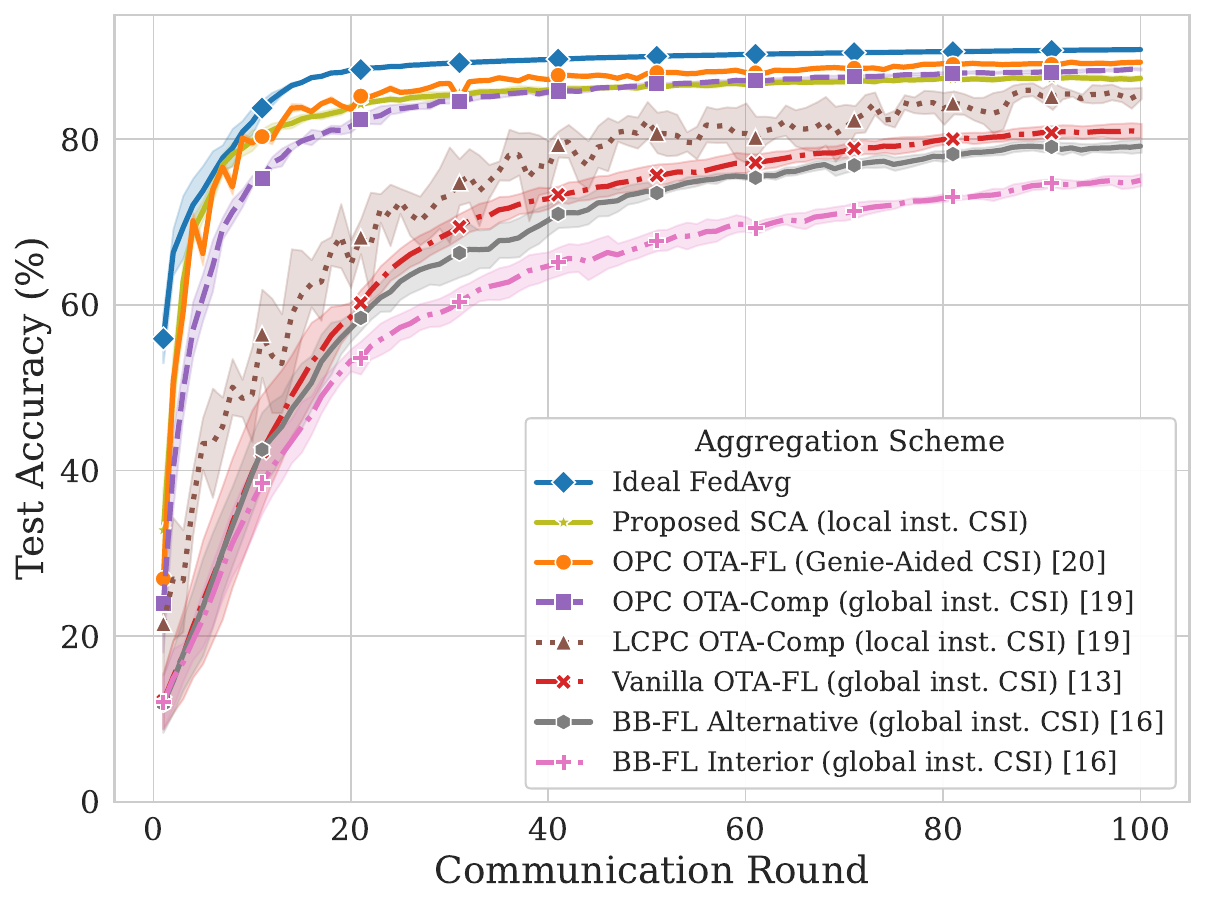}
    \caption{OTA-FL: test accuracy vs. FL rounds}
    \label{fig:sc-ota-acc}
  \end{subfigure}

  \vspace{1.5mm}

  \begin{subfigure}[t]{0.4\textwidth}
    \centering
    \includegraphics[width=\linewidth]{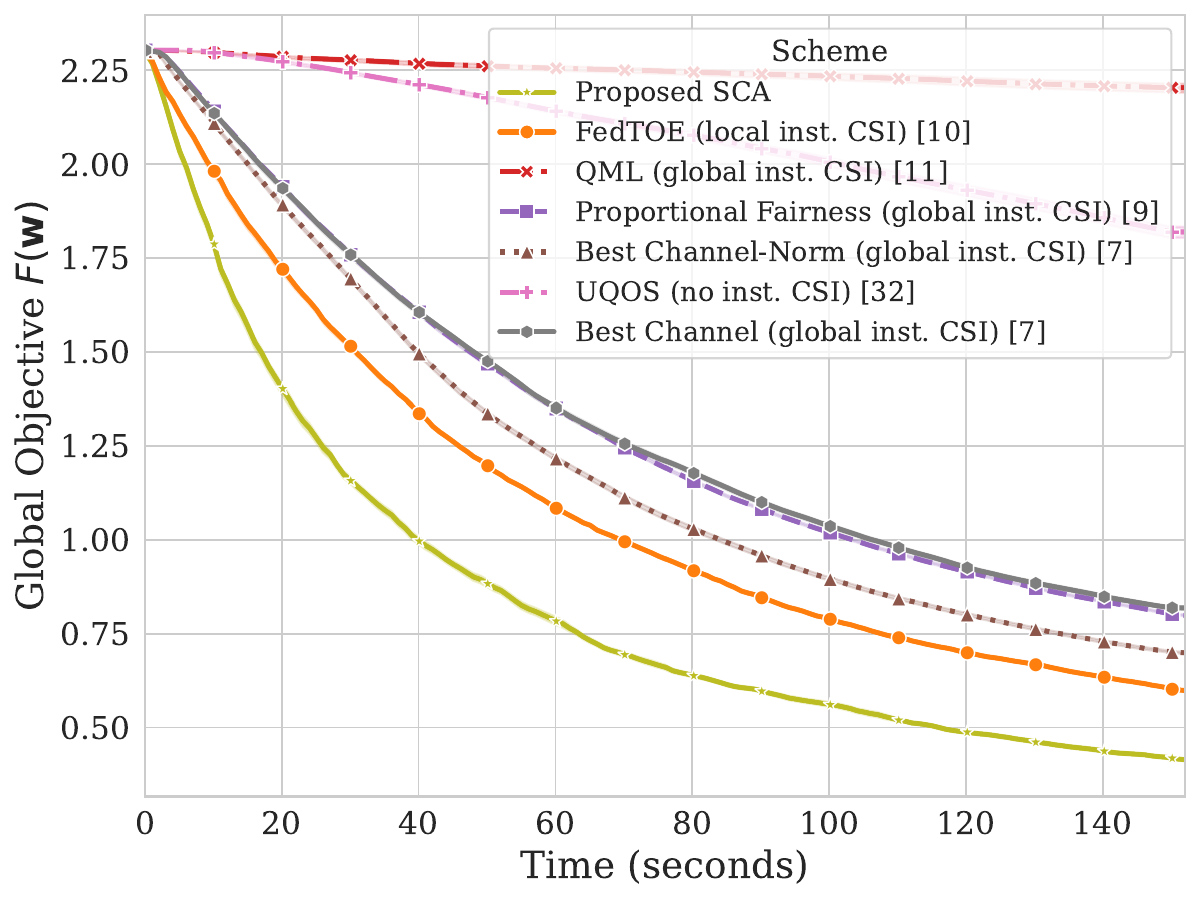}
    \caption{Digital FL: global objective $F(\mathbf{w})$  vs.\ time}
    \label{fig:sc-dig-loss}
  \end{subfigure}
  \quad 
  \begin{subfigure}[t]{0.4\textwidth}
    \centering
    \includegraphics[width=\linewidth]{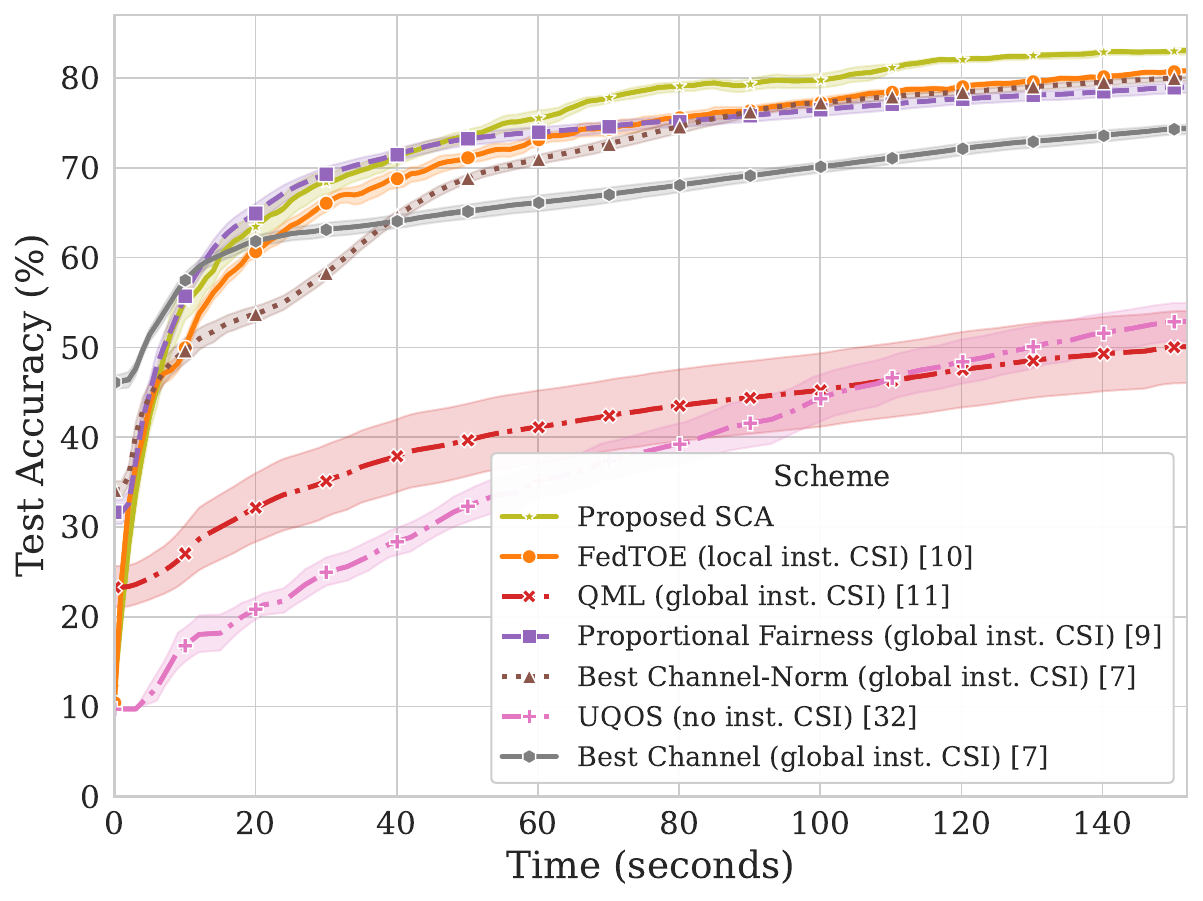}
    \caption{Digital FL: test accuracy vs.\ time}
    \label{fig:sc-dig-acc}
    \vspace{-2mm}
  \end{subfigure}
   \caption{Strongly convex task (MNIST softmax regression): OTA-FL \addb{($N = 50$ devices)} and digital FL ($N = 10$ devices) comparisons. Common parameters: $G_\text{max} = 20$, $\kappasc = 3$, and $\mu = 0.01$.}
  \label{fig:sc-grid}
  \vspace{-6mm}
\end{figure*}

In Fig. \ref{fig:sc-ota-loss} and \ref{fig:sc-ota-acc}, we compare these schemes showing the global objective $F(\mathbf{w})$ and average test accuracy vs. FL rounds, respectively, \addb{for $N = 50$ devices}. We set $\varrho_\text{in} = 0.7 \varrho_\text{max}$ for the BB-FL Interior and BB-FL Alternative. Each scheme uses a fixed step size tuned within $\eta\in(0,\tfrac{2}{\mu+L}]$. Ideal FedAvg demonstrates the best performance, since it aggregates gradients noiselessly. Among the practical wireless schemes, the best performance in terms of both metrics is attained by the genie-aided OPC OTA-FL scheme. However, this scheme requires noncausal genie-assisted CSI knowledge across all FL rounds, limiting its practicality. Our SCA-optimized OTA\text-FL closely tracks OPC OTA-FL and nearly reaches the Ideal FedAvg accuracy while requiring only statistical CSI (for design) and local instantaneous CSI (for transmission). 
While OPC OTA-Comp shows a fast global objective decay by minimizing per-round MSE performance with global CSI knowledge at the PS, the proposed scheme outperforms it despite the lack of global CSI, thanks to the well-structured bias and optimized bias-variance trade-off. Next, although LCPC OTA-Comp employs optimized truncated channel inversion, similar to the proposed scheme, it is limited by a common pre-scaler, slowing convergence. Furthermore,  BB-FL Alternative performs better than BB-FL Interior by carefully balancing the trade-off between the fraction of data exploited and maintaining less noisy FL updates, whereas BB-FL Interior restricts participation to a subset of devices, leading to poor generalization performance. Finally, although Vanilla OTA-FL eliminates model bias, it forces participation from weak-channel devices, inflating aggregation noise and slowing convergence. Overall, by judiciously designing biased average device participation to minimize the bias-variance trade-off, the proposed scheme matches the performance of the noncausal CSI-based SOTA method, while delivering noticeable performance gains over the remaining OTA-FL baselines.


%
\subsubsection{Comparison with SOTA digital-FL schemes}
Next, we benchmark the proposed digital-FL scheme against several SOTA digital-FL baselines on MNIST softmax regression (strongly convex case). For consistency, we simulate all schemes using dithered quantization.\\
$\bullet$ \underline{\textit{Best Channel}}\cite{Upd_aware_sched} selects $K\leq N$ devices with the highest instantaneous channel gain to participate in each round. The RB (time slot in our case) allocation is performed so that each device transmits the same number of overall bits.
(We exclude the scheme's gradient sparsification to ensure the numerical results are not biased by sparsity assumptions that may not hold broadly.)
 \\
$\bullet$ \underline{\textit{Best Channel-Norm}} \cite{Upd_aware_sched}, first picks $K'$ devices with the highest channel gain in an FL round, where $K \leq K' \leq N$. Then, $K$ devices with the largest local gradient norms are selected for participation out of the chosen set of $K'$ devices. Time slots are assigned proportionally to the gradient norms. 
\\
$\bullet$ \underline{\textit{Proportional Fairness}} \cite{Sched_policies} is a fairness-focused device scheduling scheme to address wireless heterogeneity. In each round $t$,
the $K\leq N$ devices with the largest normalized channel fading coefficients $\frac{\vert h_{m,t}\vert^2}{\Lambda_m}$ are selected for participation.
  \\
$\bullet$ \underline{\textit{Unbiased Quantized Optimized Scheduling (UQOS)}} \cite{Dig_vs_Analog} samples  $K\leq N$ devices in each round without replacement with probabilities $\{\pi_m\}$ obtained to minimize the convergence bound derived therein. A fixed data transmission rate $R$ is chosen for all devices, associated with an outage probability $p^\text{out}_m$, minimizing $\frac{1}{N}\sum_{m \in [N]}\frac{1}{p^\text{out}_m \pi_m}$ subject to $\pi_m \in [0,1]$ and $\sum_{m \in [N]} \pi_m = K$. Notably, the scheme accounts for both unsuccessful transmissions and device sampling and ensures that the global gradient estimate remains unbiased. \\
$\bullet$ \underline{\textit{Quantized Minimum  Latency (QML)}} \cite{Wireless_Quant_FL_Joint} aims at reducing the overall convergence time. A per-round optimization problem is solved to find the optimal bit and time slot allocation under a quantization noise variance constraint (averaged over devices). We modify the scheme to include random $K$-device sampling to obtain improved performance.
\\
$\bullet$ \underline{\textit{FL with Transmission Outage and Quantization Error}}  \underline{\textit{ (FedTOE)}} \cite{Quant_FL_Outage_Constraint} selects the transmission rate by enforcing the same outage probability $p^\text{out}_m$ for each device. $K$ devices are randomly chosen for participation. An optimization problem is solved for optimal RB and bit allocation to minimize the quantization noise variance, averaged over the devices. 

We consider $N=10$ devices in our digital-FL simulations. For computational efficiency, we cap the per-round quantization as $r_m\le 16\, , \forall m$ (c.f.\ \eqref{r_int}). The per-round latency budget for the proposed scheme is set to be $T_{\max}=0.2$~s. Since every digital-FL scheme differs in per-round latency, to give each favorable operating conditions, we set
$T_{\max}$ as 3.2, 2.1, 2.4, 3.0, 2.2 seconds for Best Channel, Best Channel–Norm, Proportional Fairness, UQOS, FedTOE, respectively. Baseline parameters $K$, $K'$, $R$, and $p^\text{out}_m$ are tuned heuristically via grid search. Since the majority of SOTA candidates use channel capacity-based transmissions, our per-round latency calculation uses channel capacity for every scheme. \addb{Importantly, because each digital scheme incurs a different per-round latency, we compare their performance vs running time (instead of rounds) for fairness. We also note that the CSI acquisition time is excluded for every method in our results, which is expected to be substantially higher in Best Channel, Best Channel–Norm, Proportional Fairness, and QML, which require global instantaneous CSI.}
Fig. \ref{fig:sc-dig-loss}-\ref{fig:sc-dig-acc} plot the global objective and test accuracy over $150$\,s of FL training. It can be observed that the proposed design performs the best among all the schemes in both metrics, achieving $\approx 83\% $ final test accuracy, thanks to the optimized device participation thresholds, post-scalers, and bit and resource allocations. Among the SOTA schemes, FedTOE performs the best by effectively guaranteeing unbiased FL updates with reduced effect of quantization errors. Proportional fairness offers a solid, low-complexity scheduling strategy to address wireless heterogeneity while ensuring zero bias, on average. Best Channel-Norm outperforms Best Channel scheduling by leveraging both CSI and local gradient strength information. 

Interestingly, despite being optimization-based (non-heuristic), UQOS and QML fail to demonstrate good performance guarantees. First, while UQOS establishes unbiased FL updates, on average, it uses uniform transmission rates across devices, forcing slower updates to accommodate the weaker channel devices. On the other hand, QML focuses on quantization noise alone, ignoring bias and transmission variance, which yields high-bias updates and therefore slower learning progress.
The proposed scheme carefully designs the digital-FL parameters by jointly considering the bias and variance terms, achieving roughly 2$\times$ faster convergence than SOTA to reach a target sub-optimality gap and accuracy.

\begin{figure*}[t]
\centering
\subfloat[\centering Test accuracy over FL rounds.]{
  \includegraphics[width=0.4\textwidth]{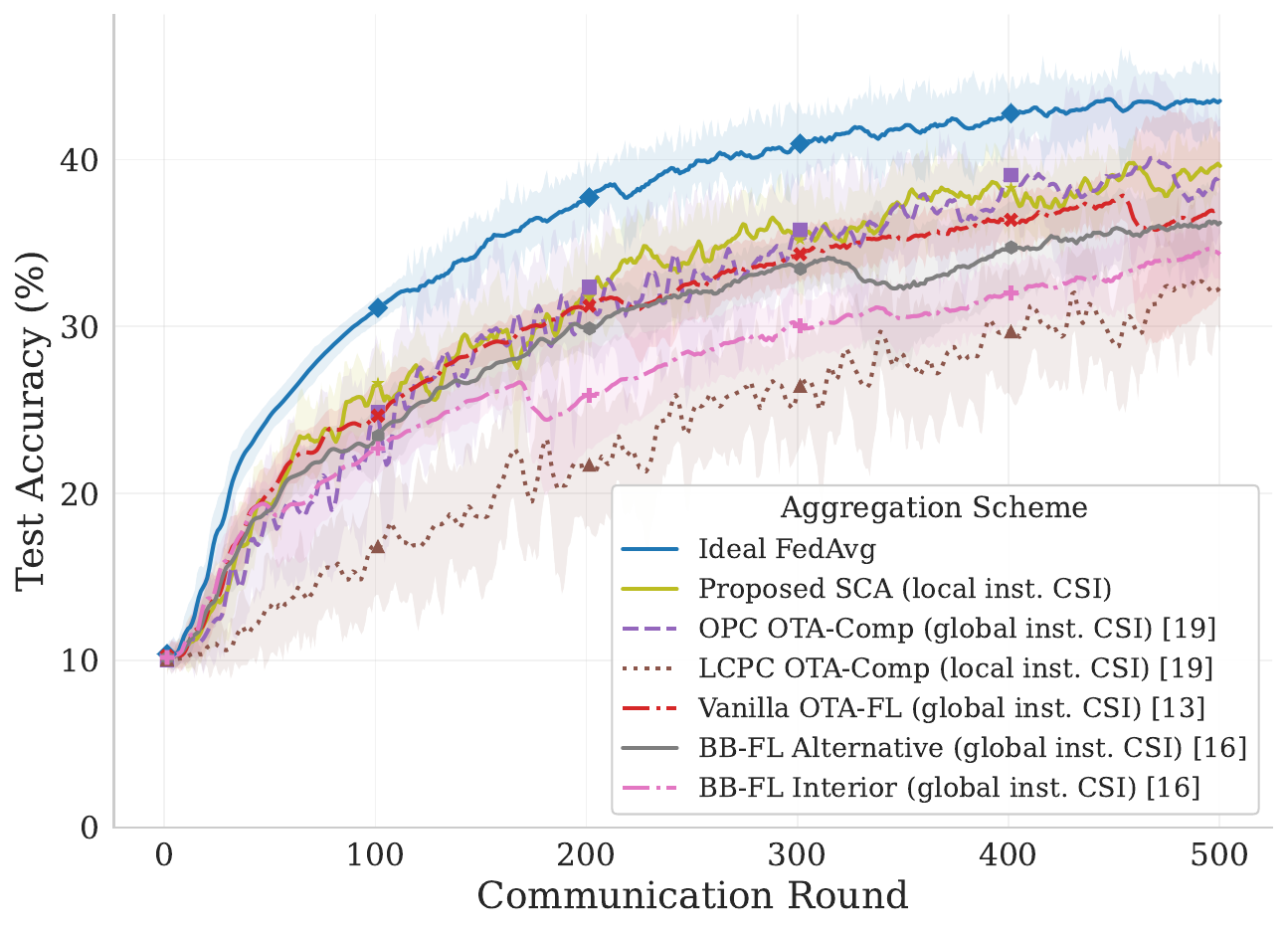}
}\qquad 
\subfloat[\centering Global objective $F(\mathbf w)$ over FL rounds.]{
  \includegraphics[width=0.4\textwidth]{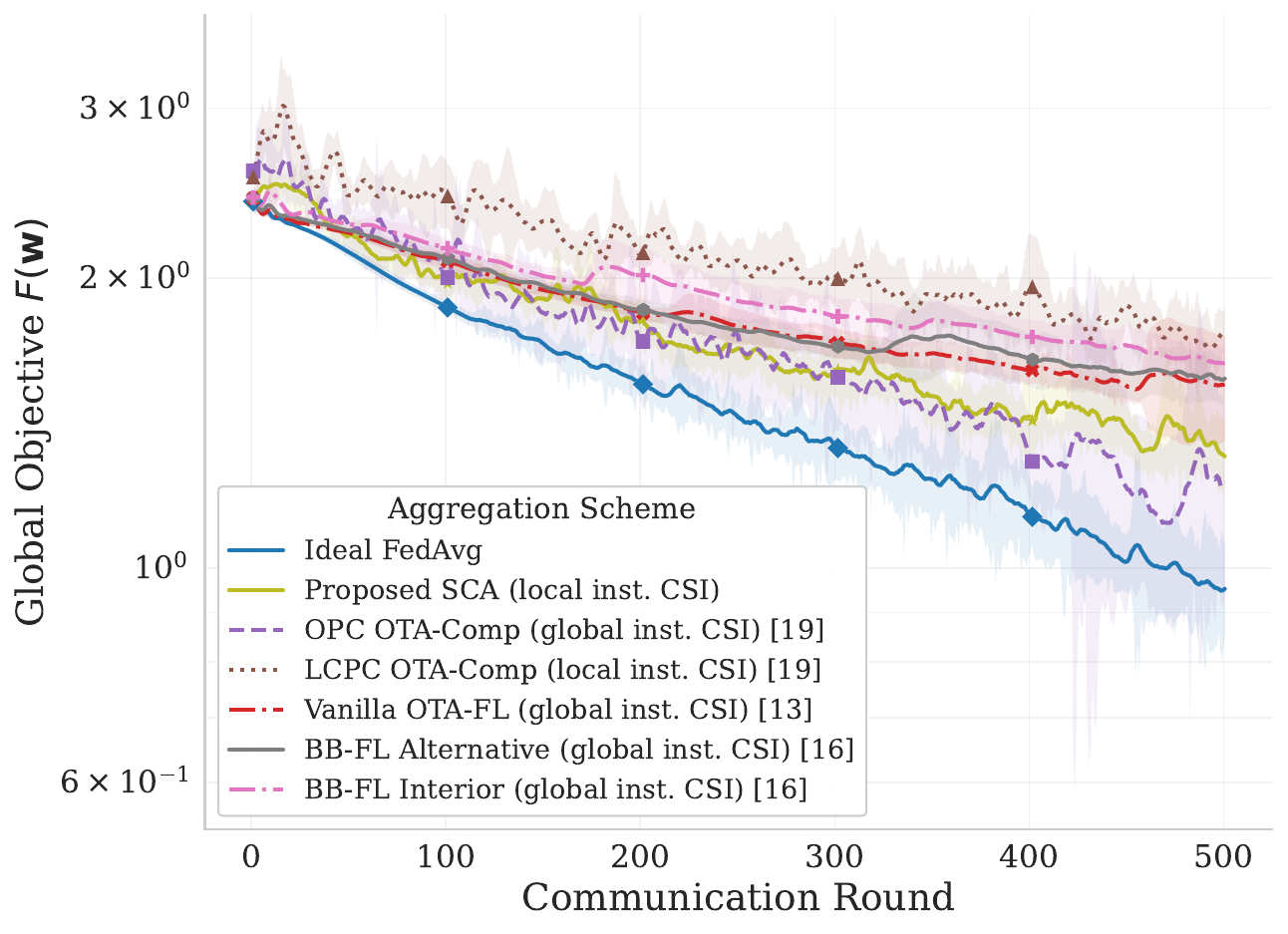}
}
\caption{\addb{Non-convex task (CIFAR-10 with ResNet-18, OTA-FL)}: $N{=}10$ devices, $G_{\max}{=}49$, $\kappa_{\mathrm{nc}}{=}2G_{\max}$, $\mu_{\mathrm{nc}}{=}0.01$.
}
\label{fig:nc-ota}
\vspace{-5mm}
\end{figure*}
\vspace{-8mm}
\addb{\subsection{Non-convex FL Task}
\label{subsec:non_convex_results}
To verify Theorem~\ref{thm:main_convergence_ncvx}, we train ResNet-18~\cite{He2015ResNet} on CIFAR-10~\cite{cifar10} dataset under OTA-FL with $N = 10$ devices, resulting in a highly non-convex FL task with $d\!\approx\!11.17$M parameters. Here, we consider \emph{two-classes per device} non-i.i.d. data split, where each device holds samples belonging to only two classes. Each device uses an $\ell_2$-regularized cross-entropy loss with regularization coefficient $\mu_{\mathrm{nc}}{=}0.01$. Due to the huge model dimensionality, we focus on the more scalable OTA-FL and compare the same baselines as in the strongly convex case, excluding OPC OTA-FL \cite{Opt_power_control_OTA_FL} (which relies on Polyak-Lojasiewicz (PL) condition inconsistent with this setting, and on knowledge of \emph{all future instantaneous CSI}). Step sizes are chosen in $\eta\!\in\!(0,1/L]$, where the smoothness parameter $L$ is estimated online from successive model updates. Figures~\ref{fig:nc-ota}(a)-(b) report test accuracy and the global objective versus rounds. Among practical FL, the proposed SCA-optimized design closely tracks OPC~OTA-Comp in accuracy, while outperforming all other practical baselines. In terms of global loss, OPC OTA-Comp is only marginally faster, leaving a small gap. 
Notably, OPC OTA-Comp requires \emph{global instantaneous CSI at the PS each round} to optimize device pre-scalers, whereas our method relies only on \emph{statistical CSI at the PS} and \emph{local instantaneous CSI at devices}. Despite this low requirement, the on-par performance confirms that optimizing the derived bias-variance trade-off yields performance gains even in the non-convex regime. Our method consistently outperforms Vanilla OTA-FL, whose zero-instantaneous-bias enforcement inflates the noise variance. While LCPC OTA-Comp also uses statistical CSI, it allows for an uncontrolled bias and hence converges sub-optimally. The proposed SCA-optimized scheme also delivers noticeable performance gains over the BB-FL Interior and Alternative scheduling strategies.}
\vspace{-4mm}
\section{Conclusion}
\label{Sec:Conclusion}
In this paper, we investigated the performance of OTA and digital FL systems in wireless heterogeneous environments \addb{for both strongly convex and smooth non-convex objectives}. Unlike existing works that either enforce zero-bias designs or allow uncontrollable bias, we proposed novel FL updates that permit a tunable fixed model bias. \addb{We characterized their learning behavior and derived unified bounds on optimality error (strongly convex) and finite-time stationarity (non-convex), revealing a bias-variance trade-off.} To prove the efficacy of our convergence analysis, we minimized this trade-off using an SCA-based parameter design optimization framework. Detailed numerical evaluations, including \addb{a non-convex FL task}, validate our theoretical findings, showing that the additional degree of freedom introduced by the tunable bias, combined with bias-variance trade-off minimization, leads to superior performance over SOTA wireless FL baselines.

 \vspace{-3mm}
\bibliographystyle{IEEEtran} 
\bibliography{Refs.bib} 
\appendices


\vspace{-5mm}
\section*{Appendix A}
\vspace{-1mm}
Here, we prove Theorem \ref{thm:main_convergence_sc} and~\ref{thm:main_convergence_ncvx},  characterizing the model optimality error (strongly convex) and finite-time expected stationarity (non-convex) for the proposed OTA- and digital-FL schemes. Auxiliary results are provided in Appendix B.

\noindent\textit{Proof of Theorem \ref{thm:main_convergence_sc}:}
We start by showing that, with $\mathcal W\equiv\{\mathbf w\in\mathbb R^d:\Vert\mathbf w\Vert\leq \max_{m\in[N]}\frac{1}{\mu} \Vert \nabla f_m(\mathbf 0) \Vert\}$, 
the minimizer of the biased problem \eqref{Incons_obj}
satisfies $\tilde{\mathbf{w}}\in\mathcal W$, and hence the 
optimality condition $\tilde{\mathbf{w}}=\mathcal{P}_{\mathcal{W}}(\tilde{\mathbf{w}}-\eta\nabla\tilde{F}(\tilde{\mathbf{w}}))$, since the unconstrained minimizer of the biased problem \eqref{Incons_obj} satisfies
$\nabla\tilde{F}(\tilde{\mathbf{w}})=\mathbf 0$.
From strong convexity, we have 
$\Vert\nabla\tilde{F}(\mathbf 0)\Vert
=
\Vert\nabla\tilde{F}(\tilde{\mathbf{w}})
-\nabla\tilde{F}(\mathbf 0)\Vert\geq\mu
\Vert\tilde{\mathbf{w}}\Vert.
$ Furthermore,
$
\Vert
\nabla\tilde{F}(\mathbf 0)\Vert
=\Vert\sum_m p_m\nabla f_m(\mathbf 0)\Vert
\leq\max_m\Vert\nabla f_m(\mathbf 0)\Vert$. Combining the two bounds, we have
$\Vert\tilde{\mathbf{w}}\Vert\leq\frac{1}{\mu}\max_m\Vert\nabla f_m(\mathbf 0)\Vert=D/2$, with $D$ defined in Theorem \ref{thm:main_convergence_sc}, implying that $\tilde{\mathbf{w}}\in\mathcal W$.

We now bound the expected FL model optimality error after $t$ rounds $\Vert \mathbf{w}_t - \mathbf{w}^*\Vert$, where $\mathbf{w}^*$ is the global minimizer in \eqref{FL_prob}. Since the iterative algorithm described in \eqref{Biased_GD} minimizes the biased objective $\tilde{F}(\mathbf{w})$ on average, we analyze the expected FL model optimality error by splitting it into two components: (1) the error between $\mathbf w_t$ and $\tilde{\mathbf{w}}$ (the biased objective minimizer), and (2) the error between $\tilde{\mathbf{w}}$ and the global minimizer $\mathbf{w}^*$, i.e., the model bias. Define $E_{t} = \|\mathbf{w}_{t} - \mathbf{w}^*\|^2$ and $\tilde{E}_{t} = \|\mathbf{w}_{t} - \tilde{\mathbf{w}}\|^2$.
Using $\Vert\mathbf a+\mathbf b\Vert^2 \le 2\Vert\mathbf a\Vert^2 + 2\Vert\mathbf b\Vert^2$, we obtain $\mathbb E[E_t]$
\begin{align}
\!\!\!{=}
\mathbb E[\|(\mathbf{w}_{t} - \tilde{\mathbf{w}}) + (\tilde{\mathbf{w}} - \mathbf{w}^*)\|^2]
\leq
2\mathbb E[\tilde E_t]{+}2\|\tilde{\mathbf{w}}{-} \mathbf{w}^*\|^2.\label{boundEt}\end{align}
We now bound the two terms $\mathbb{E} [\tilde E_{t}]$ and $\|{\tilde{\mathbf{w}}} - \mathbf{w}^*\|$.
\underline{\bf Bounding $\mathbb{E} [\tilde E_{t}]$:}
From Lemma \ref{One_step_progress_lemma:sc} (in Appendix B), with step size choice $\eta \in (0,\frac{2}{\mu + L}]$, the expected one-step FL progress satisfies
\begin{align}
\mathbb{E}[\tilde{E}_{t+1}] \le \left(1-\eta \mu\right)^2 \mathbb E[\tilde{E}_{t}] + \eta^2 \zeta,
\label{FL_one_step_prog}
\end{align}
where $\zeta$ is the global gradient estimator variance, given by Lemmas \ref{OTA_variance_lemma}-\ref{Dig_variance_lemma} for the OTA/digital schemes.
Unrolling \eqref{FL_one_step_prog} gives
\begin{align}
\mathbb{E}[\tilde{E}_{t}]
&\le \left(1-\eta \mu \right)^{2t} \tilde{E}_{0}
+ \eta^2 \zeta \sum_{j=0}^{t-1} (1- \eta \mu)^{2j} \nonumber
\\&
\overset{(a)}{\leq}
\left(1-\eta \mu\right)^{2t}\tilde{E}_{0} + \frac{\eta}{\mu} \zeta
\overset{(b)}{\leq} D^2 \left(1-\eta \mu\right)^{2t} + \frac{\eta}{\mu} \zeta,
\label{t_step}
\end{align}
where (a) uses the geometric sum and $\eta\mu\le 1$, and (b) uses $\tilde{E}_0=\|\mathbf{w}_0-\tilde{\mathbf{w}}\|^2 \le D^2$ since $\mathbf{w}_0, \tilde{\mathbf{w}} \in \mathcal{W}$ with $D =2\max_{m\in[N]} \frac{1}{\mu} \|\nabla f_m(\mathbf 0)\|$ as the diameter of $\mathcal W$.
\\
\noindent\underline{\bf Bounding $\|{\tilde{\mathbf{w}}} - \mathbf{w}^*\|$:}
Strong convexity of $\tilde{F}(\cdot)$ implies
\begin{align}
    \mu^2\|{\tilde{\mathbf{w}}}{-}\mathbf{w}^*\|^2 {\leq}\| \nabla\tilde{F}(\tilde{\mathbf{w}}){-} \nabla\tilde{F}({\mathbf{w}^*}) \|^2{=}\|  \nabla\tilde{F}({\mathbf{w}^*}) \|^2 \label{grad_norm_biased},
\end{align}
since $\nabla \tilde{F}(\tilde{\mathbf{w}}) = \mathbf{0}$. Moreover, for any arbitrary $\mathbf{w} \in \mathbb{R}^d$,
\begin{align}
&\| \nabla F(\mathbf{w}) -  \nabla\tilde{F}({\mathbf{w}})\|^2 = \big\| \sum_{m \in [N]} \big( p_m - \tfrac{1}{N} \big) \nabla f_m(\mathbf{w}) \big\|^2\nonumber\\
&= \big\| \sum_{m \in [N]} \big( p_m - \tfrac{1}{N} \big) (\nabla f_m(\mathbf{w}) -\nabla F(\mathbf{w}))\big\|^2\nonumber\\
&\le \sum_{m \in [N]} \Big(p_m - \tfrac{1}{N}\Big)^2 \cdot \sum_{m \in [N]} \| \nabla f_m(\mathbf{w})-\nabla F(\mathbf{w})\|^2,
\label{usedinnonconvx}
\end{align}
where we used $\sum_m (p_m-\frac{1}{N})=0$ in the second equality, and utilized the Cauchy–Schwarz inequality in the last step.
Evaluating at $\mathbf{w}^*$ and \addb{using the definition of $\kappasc$}, we obtain
\begin{align}
\|\nabla\tilde{F}({\mathbf{w}}^*) \| ^2 \le N \addb{\kappasc^2} \sum_{m \in [N]} \Big(p_m{-}\tfrac{1}{N}\Big)^2.
\label{grad_norm_biased1}
\end{align}
Combining \eqref{grad_norm_biased} and \eqref{grad_norm_biased1} yields
\begin{align}
\|{\tilde{\mathbf{w}}} - \mathbf{w}^*\|^2 \le \frac{N \addb{\kappasc^2}}{\mu^2} \sum_{m \in [N]}\Big( p_m - \tfrac{1}{N}\Big)^2.
\label{Bias_lemma}
\end{align}
Theorem \ref{thm:main_convergence_sc} follows by combining \eqref{t_step} and \eqref{Bias_lemma} into \eqref{boundEt}. \qed

\addb{\textit{Proof of Theorem \ref{thm:main_convergence_ncvx}:} We write
$
\nabla F(\mathbf w_t){=}\nabla \tilde F(\mathbf w_t)+\big(\nabla F(\mathbf w_t){-}\nabla \tilde F(\mathbf w_t)\big)
$. Using $\|\mathbf a+\mathbf b\|^2\le 2\|\mathbf a\|^2+2\|\mathbf b\|^2$, taking expectations, and averaging over $T$ rounds yields
\begin{align}
\label{decompose_ncvx}
\frac{1}{T}\sum_{t=0}^{T-1}\mathbb{E}[\|\nabla F(\mathbf{w}_t)\|^2] 
\le &
\frac{2}{T}\!\sum_{t=0}^{T-1}\!\mathbb E\!\big[\|\nabla \tilde F(\mathbf w_t)\|^2\big]
\\&
\!+\!\frac{2}{T}\!\sum_{t=0}^{T-1}\!\mathbb E\!\big[\|\nabla F(\mathbf w_t)-\nabla \tilde F(\mathbf w_t)\|^2\big].
\nonumber
\end{align}
Now, we bound the two key terms in \eqref{decompose_ncvx}.
\underline{\bf Bounding $\frac{2}{T}\!\sum_{t=0}^{T-1}\!\mathbb E\!\big[\|\nabla \tilde F(\mathbf w_t)\|^2\big]$:}
Invoking Lemma \ref{stionarity_lemma:nc} (in Appendix B) 
along with the definition $\tilde F(\cdot)=\sum_m p_m f_m(\cdot)$, 
 we obtain
\begin{align*}
\!\frac{2}{T}\sum_{t=0}^{T-1}\mathbb{E}[\|\nabla \tilde F(\mathbf w_t)\|^2]
\!\le\!
4\sum_m p_m\frac{f_m(\mathbf w_0){-}\mathbb E[f_m(\mathbf w_T)]}{\eta T}
{+}2\eta L\zeta.
\end{align*}
By Assumption \ref{ass:smooth_lb},  it follows that $f_n(\mathbf{w}_T)\ge f_n^{\inf}$, hence $f_n(\mathbf w_0){-}\mathbb E[f_n(\mathbf w_T)]\le f_n(\mathbf w_0)-f_n^{\inf}\le \max_{m}(f_m(\mathbf w_0){-}f_m^{\inf})$. As a result, $\frac{2}{T}\sum_{t=0}^{T-1}\mathbb{E}\big[\|\nabla \tilde F(\mathbf w_t)\|^2\big] $
\begin{align}
\le
\frac{4\,\max_{m\in[N]} \big(f_m(\mathbf w_0)-f_m^{\inf}\big)}{\eta T}
+2\eta L\,\zeta.
\label{init_term_ncvx}
\end{align}
\underline{\bf Bounding $\frac{2}{T}\!\sum_{t=0}^{T-1}\!\mathbb E\!\big[\|\nabla F(\mathbf w_t)-\nabla \tilde F(\mathbf w_t)\|^2\big]$:}
From \eqref{usedinnonconvx} (which is valid also for smooth, non-convex objectives) 
evaluated at $\mathbf{w}_t$,
combined with 
Assumption \ref{ass:bounded_data_hetr},
we obtain
\begin{align}
\nonumber
&\| \nabla F(\mathbf{w}_t) -  \nabla\tilde{F}({\mathbf{w}_t})\|^2
\le 
N\,\kappanc^2
\sum_{m \in [N]} \Big(p_m - \tfrac{1}{N}\Big)^2.
\end{align}
Thus, $\frac{2}{T}\sum_{t=0}^{T-1}\mathbb E\!\big[\|\nabla F(\mathbf w_t)-\nabla \tilde F(\mathbf w_t)\|^2\big]$
\begin{align}
&\le
2N\kappanc^2\!\!\sum_{m\in[N]}\big(p_m-\tfrac{1}{N}\big)^2.
\label{bias_term_ncvx}
\end{align}
Combining \eqref{init_term_ncvx} and \eqref{bias_term_ncvx} into \eqref{decompose_ncvx} establishes Theorem \ref{thm:main_convergence_ncvx}.}
\vspace{-1mm}

\section*{Appendix B: Auxiliary results}
\begin{lemma}[Strongly convex case: one-step progress]
Under Assumptions \ref{ass:bounded_stochastic_grad}, \ref{ass:smooth_lb}, and \ref{ass:sc} with learning step size $\eta \in (0, \frac{2}{\mu + L}]$,  the expected biased FL model optimality error after $t+1$ rounds of OTA-FL and digital FL satisfies
\begin{align*}
\mathbb{E}[\tilde{E}_{t+1}] \leq \left(1-\eta \mu\right)^2 \mathbb E[\tilde{E}_{t}] + \eta^2\zeta,
\end{align*}
where
$\zeta$ is the global gradient estimator variance 
(Lemmas \ref{OTA_variance_lemma}-\ref{Dig_variance_lemma}).
\label{One_step_progress_lemma:sc}
\end{lemma}
\begin{proof}
According to the presented generic FL model updates in \eqref{Biased_GD}, and using the fact that $\tilde{\mathbf{w}}=\mathcal{P}_{\mathcal{W}}(\tilde{\mathbf{w}}-\eta\nabla\tilde{F}(\tilde{\mathbf{w}}))$ (optimality condition for $\tilde{\mathbf{w}}$),
we have 
\begin{align*}
\tilde{E}_{t+1} 
& =\left \Vert \mathcal{P}_{\mathcal{W}}\left(\mathbf{w}_{t} - \eta \hat{\boldsymbol{g}}_t\right) - \mathcal{P}_{\mathcal{W}}(\tilde{\mathbf{w}}-\eta\nabla\tilde{F}(\tilde{\mathbf{w}}))\right\Vert^2
\\
&\leq\left\|\mathbf{w}_{t} - \eta \hat{\boldsymbol{g}}_t - (\tilde{\mathbf{w}}-\eta\nabla\tilde{F}(\tilde{\mathbf{w}}))\right\|^2, 
\end{align*}where the inequality follows from 
non-expansiveness of the projection onto the closed convex set \(\mathcal{W}\) \cite[Corollary 2.2.3]{Nesterov_book}.
Moreover, based on \eqref{OTA_exp_global_grad}, \eqref{Digital_exp_global_grad}, and Assumption \ref{ass:bounded_stochastic_grad}, the estimated global gradient $\hat{\boldsymbol{g}}_t$ in \eqref{OTA_Signal_model2} and \eqref{Digital_Signal_model2} 
satisfies
\begin{align}
    \hat{\boldsymbol{g}}_t = \sum_{m \in [N]} p_m \nabla{f}_m(\mathbf{w}_t) + \boldsymbol{e}_t = \nabla\tilde{F}(\mathbf{w}_t) + \boldsymbol{e}_t,
    \label{Est_G_Grad_decomp}
\end{align}
where $\boldsymbol{e}_t = \hat{\boldsymbol{g}}_t - \mathbb{E}[\hat{\boldsymbol{g}}_t \mid \mathbf{w}_t]$ is a zero-mean error in the gradient estimate of the biased objective $\nabla\tilde{F}(\mathbf{w}_t)$, evaluated at the current FL model $\mathbf{w}_t$. Using \eqref{Est_G_Grad_decomp}, the expected FL model optimality error at round $t{+}1$ conditional on $\mathbf{w}_t$, is derived as
\begin{align*}
 &\mathbb{E}[\tilde{E}_{t+1} \mid \mathbf{w}_t]=\big\|(\mathbf{w}_{t} - \tilde{\mathbf{w}})- \eta (\nabla\tilde{F}(\mathbf{w}_t)
 -\nabla\tilde{F}(\tilde{\mathbf{w}})+ \boldsymbol{e}_t)  \big\|^2\\
&= \big\|(\mathbf{w}_{t} - \tilde{\mathbf{w}})- \eta (\nabla\tilde{F}(\mathbf{w}_t) - \nabla\tilde{F}(\tilde{\mathbf{w}}))  \big\|^2 + \eta^2 \mathbb{E}\big[ \left\Vert\boldsymbol{e}_t\right\Vert^2 \mid \mathbf{w}_t\big],
\end{align*}
Invoking the $\mu$-strong convexity and $L$-smoothness of $\tilde{F}(\mathbf{w})$ following from Assumptions \ref{ass:smooth_lb} and \ref{ass:sc},
and bounding the gradient estimation error variance by $\zeta$, we further bound:
\begin{align*}
\mathbb{E}[\tilde{E}_{t+1} \mid \mathbf{w}_t] &\leq  \left(1-\eta \mu\right)^2 \Vert \mathbf{w}_{t} - \tilde{\mathbf{w}}\Vert^2 + \eta^2\zeta,
\end{align*}
where the contraction term follows from \cite[P2]{NCOTA}
and \cite[Sec. III]{10947567}
with $\eta \in (0, \frac{2}{\mu + L}]$. Substituting $\|\mathbf{w}_{t} - \tilde{\mathbf{w}}\|^2 = \tilde{E}_{t}$ and taking expectation over $\mathbf w_t$ concludes the proof.
\end{proof}
\vspace{-8mm}
\addb{\begin{lemma}[Non-convex case: stationarity]
Under Assumptions \ref{ass:bounded_stochastic_grad} and \ref{ass:smooth_lb} with learning step size $\eta \in (0, \frac{1}{L}]$,  the finite-time biased expected stationarity after $T$ rounds of OTA-FL and digital FL satisfies
\begin{align*}
\frac{1}{T}\sum_{t=0}^{T-1}\mathbb{E}[\|\nabla \tilde F(\mathbf w_t)\|^2]
\;\le\;
2\frac{\tilde F(\mathbf w_0)-\mathbb E[\tilde F(\mathbf w_T)]}{\eta T}
+\eta L\,\zeta,
\end{align*}
where
$\zeta$ is the global gradient estimator variance 
(Lemmas \ref{OTA_variance_lemma}-\ref{Dig_variance_lemma}). 
\label{stionarity_lemma:nc}
\end{lemma}
\begin{proof}
Utilizing $L$-smoothness of the biased objective $\tilde{F}(\cdot)$ (Assumption \ref{ass:smooth_lb}) at $\mathbf{w}_t$ and $\mathbf{w}_{t+1}$, we have $ \tilde F(\mathbf{w}_{t+1})$
\begin{align}
    \leq \tilde F(\mathbf{w}_{t}) \!+ \!\nabla \tilde{F}(\mathbf{w}_t)^T(\mathbf{w}_{t+1} \!- \!\mathbf{w}_t)\! +\! \frac{L}{2} \Vert\mathbf{w}_{t+1} - \mathbf{w}_t\Vert^2.
    \label{smoothness_F_tilde}
\end{align}
Recall that the model updates are given by $\mathbf{w}_{t+1} = \mathbf{w}_{t} - \eta \hat{\boldsymbol{g}}_t$, where $\hat{\boldsymbol{g}}_t$ is the estimate of the global gradient, with $\mathcal{W}\equiv \mathbb{R}^d$ in the non-convex case. Following the same steps as in \eqref{Est_G_Grad_decomp}, we have $\mathbb E[\hat{\boldsymbol{g}}_t|\mathbf w_t]=\nabla \tilde F(\mathbf{w}_{t})$ and  $\mathbb E[\Vert\hat{\boldsymbol{g}}_t-\nabla \tilde F(\mathbf{w}_{t})\Vert^2|\mathbf w_t]\leq \zeta$ (global gradient estimation variance, c.f. Lemmas \ref{OTA_variance_lemma}-\ref{Dig_variance_lemma}).
It then follows that 
$\mathbb E[\mathbf{w}_{t+1} \!- \!\mathbf{w}_t|\mathbf w_t]=-\eta\nabla \tilde F(\mathbf{w}_{t})$
and
$\mathbb E[ \|\mathbf{w}_{t+1} \!- \!\mathbf{w}_t\|^2|\mathbf w_t] 
\leq
\eta^2\Vert\nabla \tilde F(\mathbf{w}_{t})\Vert^2+\eta^2\zeta$. Next, we apply expectation conditional on $\mathbf{w}_t$ to both sides on \eqref{smoothness_F_tilde} and further get:
\begin{align*}
\mathbb{E}[\tilde F(\mathbf{w}_{t+1})|\mathbf w_t] &\leq \tilde F(\mathbf{w}_{t}) + (\tfrac{\eta^2L}{2} - \eta) \Vert \nabla \tilde{F}(\mathbf{w}_t)\Vert^2  + \frac{\eta^2 L}{2} \zeta \nonumber \\
& \leq \tilde F(\mathbf{w}_{t}) - \frac{\eta}{2} \Vert \nabla \tilde{F}(\mathbf{w}_t)\Vert^2  + \frac{\eta^2 L}{2} \zeta,
\end{align*}
where we utilized that the step size $\eta \leq \frac{1}{L}$ to have the second inequality. Rearranging the above inequality, 
taking total expectations, summing from $t=0$ to $T-1$, and telescoping yields the result stated in the lemma.
\end{proof}
}
\begin{proof}[Proof of Lemmas \ref{OTA_variance_lemma} and \ref{Dig_variance_lemma}]
We begin by expressing the variance of the biased OTA or digital model update as:\footnote{In this proof, all expectations are implicitly 
conditional on $\mathbf{w}_t$.} 
\begin{align}
&\mathbb{E} \big[\big\|\hat{\boldsymbol{g}}_t - \!\! \sum_{m \in [N]} \!\!p_m\nabla f_m(\mathbf w_t)\big\|^2\big] \!= \!
\mathbb{E} \big[\big\|\hat{\boldsymbol{g}}_t - \!\!\! \sum_{m \in [N]} p_m\mathbf g_{m,t}\big\|^2\big] 
\nonumber \\ & + \mathbb{E} \big[\big\| \sum_{m \in [N]} p_m(\mathbf g_{m,t} -  \nabla f_m(\mathbf w_t))\big\|^2\big],\label{variance_decompos}
\end{align}
where we used the fact that 
$\mathbb{E} [\hat{\boldsymbol{g}}_t|\mathbf g_{m,t},\forall m]= \sum_{m \in [N]} p_m\mathbf g_{m,t}$.
The second variance term is due to mini-batch data selection, and is bounded as
$\mathbb{E} [\| \sum_{m \in [N]} p_m(\mathbf g_{m,t} -  \nabla f_m(\mathbf w_t))\|^2]$ 
\begin{align}
&=  \sum_{m \in [N]} p_m^2\mathbb{E} \big[\big\Vert \boldsymbol{g}_{m,t}- \nabla f_m(\mathbf{w}_t)\big\Vert^2\big] \leq 
\sum_{m \in [N]} p_m^2 \sigma_m^2, \label{mini_batch_variance}
\end{align}
where we first used the independence of mini-batch local gradients across devices, followed by Assumption \ref{ass:bounded_stochastic_grad}. The first variance term in \eqref{variance_decompos} is due to noisy communication, and is specialized next to the two communication models.

\underline{\bf OTA-FL model:}
From \eqref{OTA_Signal_model2}, 
$\mathbb{E} [\|\hat{\boldsymbol{g}}_t - \sum_m p_m \mathbf{g}_{m,t}\|^2] $ 
\begin{align}
&=\! \mathbb{E} \Big[\Big\| \sum_m \Big(\frac{\chi^A_{m,t}\gamma_{m}}{\alpha}-p_m\Big)\mathbf g_{m,t} + \frac{\mathbf{z}_t}{\alpha}\Big\|^2\Big] \nonumber\\
&\overset{(a)}{=}  \sum_{m \in [N]} \mathbb{E}\Big[\Big(\frac{\chi^A_{m,t}\gamma_{m}}{\alpha}-p_m\Big)^2\Big] \mathbb{E}\Big[\Vert\boldsymbol{g}_{m,t}\Vert^2\Big] + \frac{d N_0}{\alpha^2}
\nonumber \\
&\overset{(b)}{\leq} \sum_{m \in [N]}  G_\text{max}^2 p_m^2 \left(\frac{\gamma_m}{\alpha_m} -1 \right)  + \frac{d N_0}{\alpha^2},
\label{OTA_comm_variance}
\end{align}
where in (a) we used $\mathbb{E}[\frac{\chi^A_{m,t}\gamma_{m}}{\alpha}] = p_m$ and the mutual independence of
noise, fading (\(\chi^A_{m,t}\)) and mini-batch local gradients across the devices.
(b) follows from $\Vert\mathbf g_{m,t}\Vert\leq G_\text{max},\forall m,t$ (Assumption \ref{ass:bounded_loss_grad}) and 
$p_m = \frac{\alpha_m}{\alpha}$.
Finally, using \eqref{OTA_comm_variance} and \eqref{mini_batch_variance} with \eqref{variance_decompos} provides the desired result.

\underline{\bf Digital FL model:}
From \eqref{Digital_Signal_model2}, the first term in \eqref{variance_decompos} is bounded as $\mathbb{E} [\|\hat{\boldsymbol{g}}_t - \sum_m p_m \mathbf{g}_{m,t}\|^2] $ 
\begin{align}
&= \mathbb{E} \Big[\Big\Vert \sum_{m \in [N]} \frac{\chi^D_{m,t}
\boldsymbol{g}_{m,t}^q
}{\nu_m} - p_m \boldsymbol{g}_{m,t}\Big\Vert^2 \Big]
\nonumber\\
&\overset{(a)}{=}  \sum_{m \in [N]} \mathbb{E} \Big[\Big\Vert\frac{\chi^D_{m,t}
\boldsymbol{g}_{m,t}^q
}{\nu_m} - p_m \boldsymbol{g}_{m,t}\Big\Vert^2 \Big]\nonumber \\
&=  \sum_{m \in [N]} \!\mathbb{E} \Big[\Big(\frac{\chi^D_{m,t}}{\nu_m}\Big)^2\Big]\mathbb{E} \Big[
\Vert\boldsymbol{g}_{m,t}^q\Vert^2 \Big]\!- \!p_m^2 \mathbb{E} [\Vert\boldsymbol{g}_{m,t}\Vert^2]\nonumber\\
&\overset{(b)}{\leq} \sum_{m \in [N]}\frac{\beta_m \mathbb{E}\left[
\frac{d\Vert\boldsymbol{g}_{m,t}\Vert_\infty^2}{(2^{r_m}-1)^2} 
 + \Vert\boldsymbol{g}_{m,t}\Vert^2
 \right]}{\nu_m^2 }\!- \!p_m^2 \mathbb{E} [\Vert\boldsymbol{g}_{m,t}\Vert^2]\nonumber,
\end{align}
where (a) follows from $\mathbb{E} [\frac{\chi^D_{m,t}
\boldsymbol{g}_{m,t}^q
}{\nu_m}] = p_m\boldsymbol{g}_{m,t}$,
the unbiasedness of dithered quantization, and the independence of fading and mini-batch gradients across devices.
(b) follows from the bound on the error of dithered quantization 
(see
\cite{Quant_FL_Outage_Constraint} and references therein):
$
\mathrm{var}(\boldsymbol{g}^q_{m,t}|\boldsymbol{g}_{m,t})
\leq
\frac{d\Vert \boldsymbol{g}_{m,t}\Vert_\infty^2}{(2^{r_m}-1)^2}$.
Leveraging Assumptions  \ref{ass:bounded_loss_grad} and \ref{ass:bounded_stochastic_grad}, along with 
$\beta_m\leq1$ and $p_m=\beta_m/\nu_m$, 
and using the fact that $\Vert\boldsymbol{g}_{m,t}\Vert_\infty \leq\Vert\mathbf g_{m,t}\Vert\leq G_\text{max}$ (Assumption~\ref{ass:bounded_loss_grad}), we further obtain $\mathbb{E} [\|\hat{\boldsymbol{g}}_t - \sum_m p_m \mathbf{g}_{m,t}\|^2] $

\begin{align}
&\leq \sum_{m \in [N]} p_m^2G_\text{max}^2\Big(\frac{1}{\beta_m} -1
 +\frac{d}{\beta_m(2^{r_m}-1)^2}
\Big). \label{Dig_comm_variance}
\end{align}
Finally, using \eqref{Dig_comm_variance} and \eqref{mini_batch_variance} with \eqref{variance_decompos} completes the proof. \end{proof}
 



\vspace{-5mm}\begin{IEEEbiography}[{\includegraphics[width=1in,height=1.25in,clip,keepaspectratio]{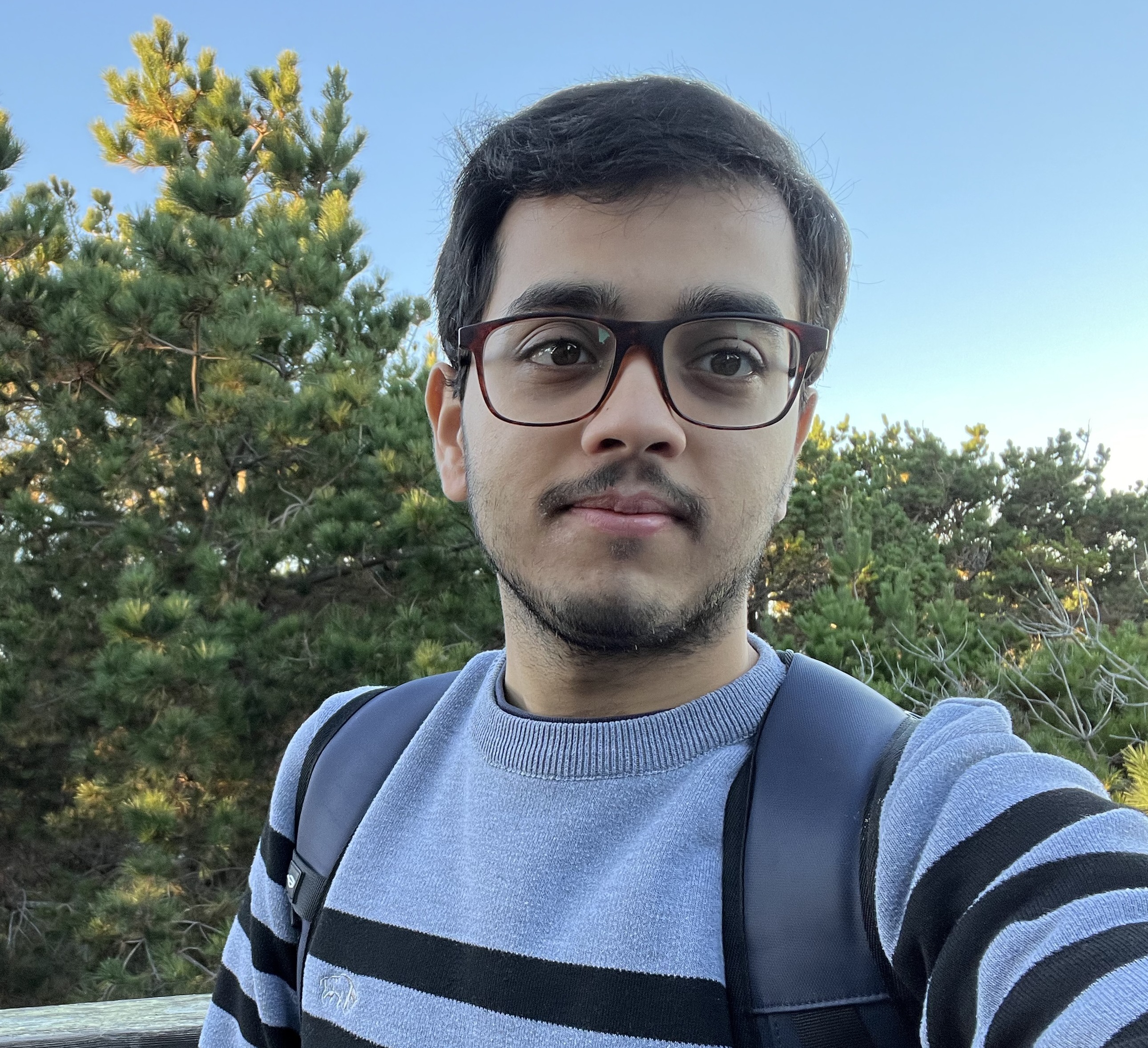}}]{Muhammad Faraz Ul Abrar}
(Graduate Student Member, IEEE) received the bachelor’s degree in electrical engineering from the School of Electrical Engineering and Computer Science, National University of Sciences and Technology (NUST), Islamabad, Pakistan, in 2021. He is currently working toward his Ph.D. degree with Arizona State University (ASU), USA. His current research directions include federated learning over wireless networks, distributed optimization, and time-varying optimization.
\end{IEEEbiography}

\begin{IEEEbiography}[{\includegraphics[width=1in,height=1.25in,clip,keepaspectratio]{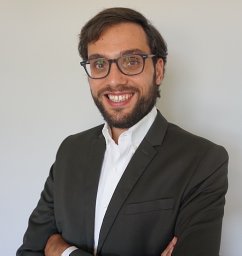}}]{Nicol\`{o} Michelusi}(Senior Member, IEEE) received the B.Sc. (with honors), M.Sc. (with honors), and Ph.D. degrees from the University of Padova, Italy, in 2006, 2009, and 2013, respectively, and an M.Sc. degree in telecommunications engineering from the Technical University of Denmark, Denmark, in 2009. 
From 2013 to 2015, he was a Postdoctoral Research Fellow with the Ming Hsieh Department of Electrical Engineering, University of Southern California, Los Angeles, CA, USA. From 2016 to 2020, he was an Assistant Professor with the School of Electrical and Computer Engineering, Purdue University, West Lafayette, IN, USA. He is currently an Associate Professor with the School of Electrical, Computer and Energy Engineering, Arizona State University, Tempe, AZ, USA. 
His research interests include 5G wireless networks, millimeter-wave communications, stochastic optimization, and decentralized and federated learning over wireless systems. 
He served as an Associate Editor for the \emph{IEEE Transactions on Wireless Communications} (2016–2021) and for the \emph{IEEE Transactions on Communications} (2023-2025). He is also a member of the IEEE Signal Processing for Communications and Networking Technical Committee. 
He co-chaired the Distributed Machine Learning and Fog Networking Workshop at IEEE INFOCOM in 2021, 2023, and 2024; the Wireless Communications Symposium at IEEE GLOBECOM 2020; the IoT, M2M, Sensor Networks, and Ad-Hoc Networking Track at IEEE VTC 2020; and the Cognitive Computing and Networking Symposium at ICNC 2018. He served as Technical Area Chair for the Communication Systems track at Asilomar 2023. 
He is the recipient of several awards, including the NSF CAREER Award in 2021, the IEEE Communication Theory Technical Committee (CTTC) Early Achievement Award in 2022, the IEEE Communications Society William R. Bennett Prize in 2024, and the IEEE ICC Best Paper Award for the Communication Theory Symposium in 2025.
\end{IEEEbiography}

\end{document}